\documentclass[sts]{imsart}
\input{commands.tex}

\usepackage{graphicx}
\usepackage[toc,page,titletoc]{appendix}
\usepackage{hyperref}

\begin{document}
\begin{frontmatter}
\title{Sinkhorn EM: An Expectation-Maximization algorithm based on entropic optimal transport}

\runtitle{Sinkhorn EM}

\begin{aug}
\normalsize
  
  Gonzalo Mena \textsuperscript{1}, Amin Nejatbakhsh \textsuperscript{2,3,4},
  Erdem Varol \textsuperscript{2,3,4}, Jonathan Niles-Weed \textsuperscript{5,6} \par \bigskip

  \textsuperscript{1}Department of Statistics and Data Science Initiative, Harvard University.
  \textsuperscript{2}Department of Neuroscience and Statistics, Columbia University.\par
  \textsuperscript{3}Grossman Center for the Statistics of Mind, Columbia University. \par
  \textsuperscript{4}Zuckerman Institute, Center for Theoretical Neuroscience, Columbia University. \par
  \textsuperscript{5} Courant Institute of Mathematical Sciences and the Center for Data Science, New York University.\par
  \textsuperscript{6} Institute for Advanced Study, Princeton.
  


\address{{Gonzalo Mena}\\
{Department of Statistics} \\
{Harvard University}\\
{1 Oxford St \#7,}\\
{Cambridge, MA 02138-4307, USA}\\
{gomena@fas.harvard.edu}
}

\address{{Amin Nejatbakhsh}\\
{Jerome L. Greene Science Center} \\
{Columbia University}\\
{3227 Broadway}\\
{New York, NY 10027, USA}
}

\address{{Erdem Varol}\\
{Jerome L. Greene Science Center} \\
{Columbia University}\\
{3227 Broadway}\\
{New York, NY 10027, USA}
}

\address{{Jonathan Niles-Weed} \\ 
{Warren Weaver Hall \#1123} \\ 
{Center for Data Science \#603}\\
{New York University}\\
{New York, NY, USA, 10016}
}

\runauthor{Mena, Nejatbakhsh, Varol, Niles-Weed}
\end{aug}

\begin{abstract}
We study Sinkhorn EM (sEM), a variant of the expectation-maximization (EM) algorithm for mixtures based on entropic optimal transport.
sEM differs from the classic EM algorithm in the way responsibilities are computed during the expectation step: rather than assign data points to clusters independently, sEM uses optimal transport to compute responsibilities by incorporating prior information about mixing weights.
Like EM, sEM has a natural interpretation as a coordinate ascent procedure, which iteratively constructs and optimizes a lower bound on the log-likelihood.
However, we show theoretically and empirically that sEM has better behavior than EM: it possesses better global convergence guarantees and is less prone to getting stuck in bad local optima.
We complement these findings with experiments on simulated data as well as in an inference task involving \textit{C.~elegans} neurons and show that sEM learns cell labels significantly better than other approaches.
\end{abstract}

\begin{keyword}[class=AMS]
\kwd[]{Statistics}
\end{keyword}
\begin{keyword}[class=KWD]
EM Algorithm, Mixture of Gaussians, Optimal Transport, Entropic Regularization. 
\end{keyword}

\end{frontmatter}

\maketitle

\tableofcontents

\section{Introduction}
The expectation-maximization (EM) algorithm \citep{Dempster1977} is a fundamental method for maximum-likelihood inference in latent variable models.
Though this maximization problem is generally non-concave, the EM algorithm is nevertheless a popular tool which is often effective in practice.
A great deal of recent work has therefore focused on finding provable guarantees on the EM algorithm and on developing modifications of this algorithm which perform better in practice.

In this work, we develop a new variant of EM, which we call \emph{Sinkhorn EM} (sEM), which has significant theoretical and practical benefits when learning mixture models when prior information about the mixture weights is known.
Recent theoretical findings~\citep{Xu2018} indicate that incorporating prior information of this kind into the standard EM algorithm leads to poor convergence.
\Citet{Xu2018} therefore suggest ignoring information about the cluster weights altogether, a procedure they call ``overparameterized EM.''
While overparameterized EM gets stuck less often than vanilla EM, our experiments show that it converges significantly more slowly.
By contrast, Sinkhorn EM offers a practical and theoretically justified ``best of both worlds'': it enjoys better theoretical guarantees than vanilla EM, seamlessly incorporates prior knowledge, and has better performance on both synthetic and real data than both vanilla and overparameterized EM.

We define Sinkhorn EM by replacing the log-likelihood by an objective function based on entropic optimal transport (OT).
This new objective has the same global optimum as the negative log likelihood but has more curvature around the optimum, which leads to faster convergence in practice.
Unlike the standard EM algorithm, which obtains a lower bound for the log-likelihood by computing the posterior cluster assignment probabilities for each observation independently, Sinkhorn EM computes cluster assignments which respect the known mixing weights of each components.
Computing these assignments can be done efficiently by \emph{Sinkhorn's algorithm}~\citep{Sinkhorn1967}, after which our procedure is named.

\paragraph{Our contributions}
\begin{itemize}
\item We define a new loss function for learning mixtures based on entropic OT, and show that it is consistent in the population limit and has better geometrical properties than the log likelihood (Section~\ref{sec:otloss}).
\item We give a simple EM-type procedure to optimize this loss based on Sinkhorn's algorithm, and prove that it is less prone to getting stuck in local optima than vanilla EM (Sections~\ref{sec:em} and \ref{sub:assy}).
\item We show on simulated (Section~\ref{sec:simulation}) and \textit{C.~elegans} data (Section~\ref{sec:c_elegans}) that sEM converges in fewer iterations and recovers cluster labels better than either vanilla or overparameterized EM.
\end{itemize}

Proofs our our theoretical results and additional experiments appear in the appendices.

 \subsection{Related work}
The EM algorithm has been the subject of a vast amount of study in the statistics and machine learning community~\citep[see][for a comprehensive introduction]{McLKri08}.
Our work fits into two lines of work on the subject.
First, following~\citet{Neal1998}, we understand the EM algorithm as one of a family of algorithms which maximizes a lower bound on the log-likelihood function via an alternating procedure.
This perspective links EM to variational Bayes methods~\citep{TziLikGal08,Blei2017} and provides the starting point for natural modifications of the EM algorithm with better computational properties~\citep{CapMou09}.
The sEM algorithm fits naturally into this framework, as a procedure which imposes the additional constraint during the alternating procedure that the known mixing weights are preserved.
Second, our work fits into a recent line of work~\citep{Xu2018,DasTzaZam17,BalWaiYu17}, which seeks to obtain rigorous convergence guarantees for EM-type algorithms in simplified settings.

We also add to the literature on connections between Gaussian mixtures and entropic OT.
\Citet{Rigollet2018} showed that maximum-likelihood estimation for Gaussian mixture models is equivalent to an optimization problem involving entropic OT, under a restrictive condition on the set of measures being considered.
Proposition~\ref{mini} shows that this condition can be removed in the population limit.

Several prior works have advocated for the use of OT-based procedures for clustering tasks.
A connection between entropic OT and EM was noted by \citet{Papadakis2017}; however, they did not propose an algorithm and did not consider the optimization questions we focus on here.
\Citet{Genevay2019} suggested using entropic OT to develop differentiable clustering procedures suitable for use with neural networks.
Our approach differs from theirs in that we focus on an elementary alternating minimization algorithm rather than on deep learning.

In recent prior work, a subset of the authors of this work used sEM to perform a joint segmentation and labeling task on \emph{C.~elegans} neural data, but without theoretical support~\citep{preprint}.
In this work, we formalize and justify the sEM proposal.
\subsection{Entropic optimal transport}
\label{sub:background}
In this section, we briefly review the necessary background on entropic OT.
Let $P$ and $Q$ be two Borel probability measures on Radon spaces $\cX$ and $\cY$, respectively. Given a cost function $c: \cX \times \cY \to \RR$, we define the entropy-regularized optimal transport cost between $P$ and $Q$ as
\begin{equation}\label{eq:sinkdef}
S(P,Q) := \inf_{\pi\in \Pi(P,Q)} \left[\int_{\cX \times \cY}  c(x,y)\,\dd\pi(x,y) +  H(\pi\lvert P\otimes Q)\right]\,,
 \end{equation}
where $\Pi(P,Q)$ is the set of all joint distributions with marginals equal to $P$ and $Q$, respectively, 
and $H(\alpha\lvert \beta)$ denotes the relative entropy between probability measures $\alpha$ and $\beta$ defined as $\int \log \frac{\dd\alpha}{\dd\beta}(x)\dd\alpha(x)$ if $\alpha \ll \beta$ and $+ \infty$ otherwise.
 
The optimization problem~\eqref{eq:sinkdef} is convex and can be solved by an efficient algorithm due to~\citet{Sinkhorn1967}, which was popularized in the machine learning community by~\citet{Cuturi2013}.
The fact that approximate solutions to~\eqref{eq:sinkdef} can be found in near linear time~\citep{Altschuler2017} forms the basis for the popularity of this approach in computational applications of OT~\citep{Peyre2019}.

\section{The entropic-OT loss as an alternative to the 
log-likelihood}\label{sec:otloss}
In this section, we define the basic loss function optimized by sEM and compare it to the negative log-likelihood.
We show that, in the population limit, these two losses have the same global minimum at the true parameter; however, the entropic OT loss always dominates the negative log-likelihood and has strictly more curvature at the minimum.
These findings support the claim that the entropic OT loss has better local convergence properties than the negative log-likelihood alone.

We recall the basic setting of mixture models.
We let $X$ and $Y$ be random variables taking values in the space $\cX \times \cY$, with joint distribution
\begin{equation}\label{defqxy} \dd Q_{X, Y}^\theta(x, y) = e^{-g^\theta(x, y)}\dd P_0(x) \dd \mu_Y(y)\,,
\end{equation}
where $P_0$ is a known prior distribution on $\cX$, $\mu$ represents a suitable base measure on $\cY$ (e.g., the Lebesgue measure when $\cY = \RR^d$), and $\{g^\theta\}_{\theta \in \Theta}$ is some family of functions which satisfies the requirement that $e^{-g^\theta(x, y)} \dd \mu(y)$ is a probability measure on $\cY$ for each $x \in \cX$ and $\theta \in \Theta$.
We write $q^{\theta}(x, y)$ for the density $\frac{\dd Q_{X, Y}^\theta}{\dd P_0 \otimes \dd \mu}$ and write $Q_{Y}^\theta$ and $q_Y^\theta$ for the marginal law and density of $Y$.

This definition encapsulates many common scenarios.
For example, for mixtures of Gaussian with known mixing weights, $\cX$ acts as an index space and $P_0$ represents the weighting of the components, while the parameter $\theta$ encapsulates the mean and covariance of each component.

We assume that we are in the well-specified case where $(X, Y) \sim Q^{\theta^*}$, and consider the problem of estimating $\theta^*$.
Standard maximum-likelihood estimation consists in minimizing the negative log-likelihood $\ell(\theta) := - \E_{Y \sim Q^{\theta^*}} \log q^\theta(Y)$ or its finite-sample counterpart, $\hat\ell(\theta) := - \frac{1}{n} \sum_{i=1}^n \log q^\theta(Y_i)$, where $Y_1, \dots, Y_n$ are i.i.d.~observations.

As an alternative to the log-likelihood, we define the following entropic-OT loss:
\begin{equation}\label{eq:otloss}
L(\theta):=S_\theta(P_0, Q_Y^{\theta^*}) := \inf_{\pi\in \Pi(P_0,Q^{\theta^*}_Y)} \left[\int  g^\theta(x, y)\,\dd\pi(x,y) +  H(\pi\lvert P_0 \otimes Q_Y^{\theta^*})\right]\,.
\end{equation}
Here, we define $S_\theta(P, Q)$ to be the value of the entropic OT problem~\eqref{eq:sinkdef} with $\theta$-dependent cost $c(x, y) = g^\theta(x, y)$.
This object likewise has a natural finite-sample analogue:
\begin{equation*}
\hat L(\theta) := S_\theta\Big(P_0, \frac 1n \sum_{i=1}^n \delta_{Y_i}\Big)\,,
\end{equation*}
where the true distribution of $Y$ has been replaced by the observed empirical measure.

In the following proposition we show that the entropic OT loss always dominates the negative log-likelihood, and moreover, that in the population limit these two functions are both minimized at the true parameter.

\begin{proposition} \label{mini}
Let $\mu_Y$ be any probability measure on $\cY$.
Then for all $\theta \in \Theta$,
\begin{equation}\label{eq:dominate}
S_\theta(P_0, \mu_Y) \geq - \E_{Y \sim \mu_Y} \log q^{\theta}(Y)\,.
\end{equation}
In particular, for all $\theta \in \Theta$,
\begin{equation}\label{eq:ell_dominate}
L(\theta) \geq \ell(\theta)  \quad \text{and} \quad \hat L(\theta) \geq \hat \ell(\theta)\,.
\end{equation}
Moreover, if $(X, Y)$ has distribution $Q_{X, Y}^{\theta^*}$, then $L(\theta^*) = \ell(\theta^*)$ and $L(\cdot)$ is minimized at $\theta^*$.
\end{proposition}
To prove Proposition~\ref{mini}, we recall the definition of the $F$-functional due to~\citet{Neal1998}.
For a fixed $y \in \cY$, probability measure $\tilde P$ on $\cX$, and parameter $\theta \in \Theta$, we write
\begin{equation}\label{eq:f_def}
F_{y}(\tilde P, \theta) \defeq \E_{X\sim \tilde P}\left(\log \frac{\dd Q_{X, Y}^\theta}{\dd P_0 \otimes \dd \mu}(X, y)\right)-E_{X\sim \tilde P}\left(\log \frac{\dd \tilde P}{\dd P_0}(X)\right)\,.
\end{equation}
\Citet{Neal1998} show that $
\E_{Y \sim \mu_Y}  \log q^\theta(Y) = \max_{P} \E_{Y \sim \mu_Y} F_Y(P(\cdot | Y), \theta)$,
where the maximization is taken over all transition kernels.
In the proof of Proposition~\ref{mini}, we show that
\begin{equation*}
- S_\theta(P_0, \mu_Y) = \max_{P \in \bar{\Pi}(P_0, \mu_Y)} \E_{Y \sim \mu_Y} F_Y(P(\cdot | Y), \theta)\,,
\end{equation*}
where now the maximization is taken over the smaller set of kernels $P$ which satisfy $P_0 = \int P( \cdot | y) \dd \mu_Y(y)$.
The inequality~\eqref{eq:dominate} follows easily, and the remaining claims are simple corollaries.

Proposition~\ref{mini} suggests it might be sensible to use $L(\theta)$ as an alternative to the negative log-likelihood in inference tasks. Indeed, the fact that $L(\theta)$ always dominates the negative log-likelihood also suggests that at the optimum it has a higher curvature, favoring optimization. In the following proposition we show this is the case, at least around the global optimum.

\begin{proposition}\label{prop:curvature}
If $(X, Y) \sim Q_Y^{\theta^*}$, then, $\nabla^2 L(\theta^*) \succ \nabla^2 \ell(\theta^*)$.
\end{proposition}
A proof appears in the appendix.

\begin{figure}[t]

\includegraphics[width=1.0\textwidth]{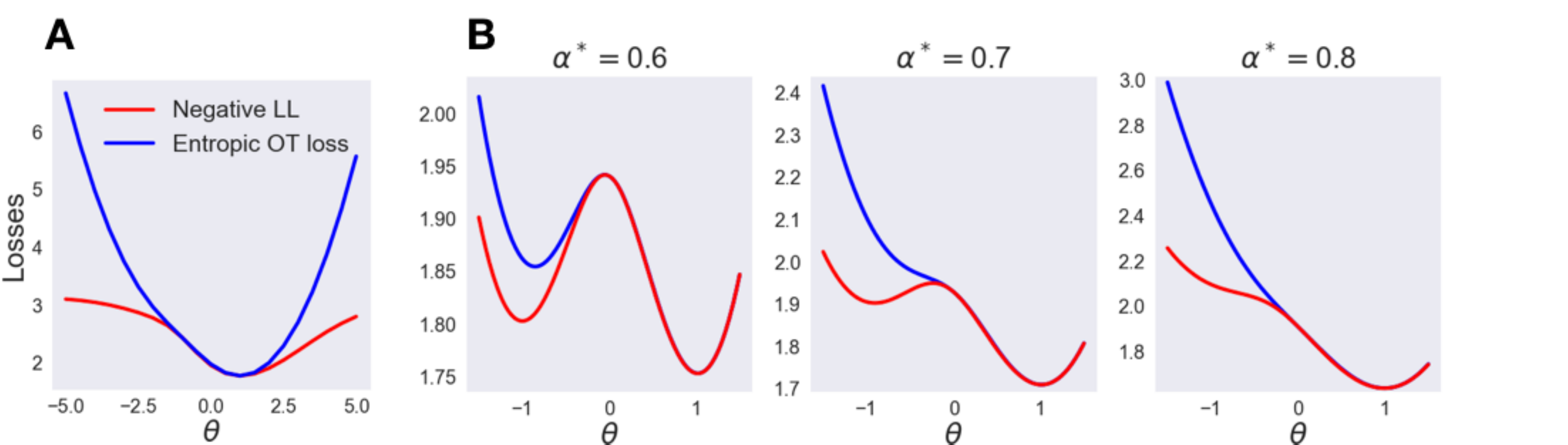}
\caption{Qualitative comparison between the log-likelihood and entropic OT loss. \textbf{A} Around the local optima in the model \eqref{eq:mixtutwo} (same as Fig. \ref{fig:example2}) the entropic OT loss (i) dominates the negative log-likelihood (Proposition \ref{mini}) and (ii) has more curvature at the minimum (Proposition \ref{prop:curvature}). \textbf{B} In the model \eqref{eq:mixtureg} (same as Fig. \ref{fig:example1}) Sinkhorn OT may have fewer bad local optima (Theorem \ref{teo:mixture}).}
\label{fig:global}
\end{figure}
\section{Sinkhorn EM: An EM-type algorithm based on entropic optimal transport}\label{sec:em}

The results of Section \ref{sec:otloss} suggest $L(\cdot)$ might be better suited than the log-likelihood to first-order methods, which, in the case of the log-likelihood include the EM algorithm \citep{Redner1984, Xu199}. In this section, we show that the entropic loss gives rise to a practical EM-type algorithm, which we call Sinkhorn-EM (sEM).

Our starting point is the observation that the EM algorithm can be understood as a maximimization-maximization procedure on measure and parameter spaces \citep{Neal1998,CsiszarT1984}. 
Recall the functional $F$ introduced in~\eqref{eq:f_def}.
\Citet{Neal1998} show that the standard EM algorithm can be written as follows.
\paragraph{Standard EM:}
\begin{itemize}
\item E-step: Let $P^{(t+1)} = \argmax_P \E_{Y \sim \mu_Y} F_Y(P(\cdot | Y), \theta^{(t)})$.
\item M-step: Let $\theta^{(t+1)} = \argmax_\theta \E_{Y \sim \mu_Y} F_Y(P^{(t+1)}(\cdot | Y), \theta)$.
\end{itemize}
With this variational perspective in mind we see the EM algorithm is simply the alternate maximization of $F$ over $P$ and $\theta$, whose (local) convergence is guaranteed by virtue of classical results \citep{CsiszarT1984,Gunawardana2005,Zangwill1969}.

The proof of Proposition~\ref{mini} shows that the entropic OT loss is obtained by restricting the optimization over $P$ to the set $\bar{\Pi}(P_0, \mu_Y)$ of kernels satisfying $P_0 = \int P(\cdot | Y) \dd \mu_Y(Y)$.
By incorporating this constraint into the alternating maximization procedure, we obtain sEM.
\paragraph{Sinkhorn EM:}
\begin{itemize}
\item E-step: Let $P^{(t+1)} = \argmax_{P \in \bar{\Pi}(P_0, \mu_Y)} \E_{Y \sim \mu_Y} F_Y(P(\cdot | Y), \theta^{(t)})$.
\item M-step: Let $\theta^{(t+1)} = \argmax_\theta \E_{Y \sim \mu_Y} F_Y(P^{(t+1)}(\cdot | Y), \theta)$.
\end{itemize}

The following theorem gives an implementation of this method, and shows that it always makes progress on $L$.
\begin{theorem}\label{teo:em}
The sEM method is equivalent to the following.
\begin{itemize}
\item E-step: Let $\pi^{(t+1)} =\argmin_{\pi\in\Pi(P_0,\mu_Y)} \int g^{\theta^{(t)}}(x, y) \dd \pi(x, y) + H(\pi|P_0 \otimes \mu_Y)$.
\item M-step: let $\theta^{(t+1)} = \argmin_\theta \int g^{\theta}(x, y) \dd \pi^{(t+1)}(x, y)$.\end{itemize}
Moreover, in the infinite-sample setting $\mu_Y = Q^{\theta^*}_Y$, the sequence $\{L(\theta^{t})\}$ is nonincreasing and $L(\theta^{t+1})<L(\theta^t)$ if $\theta^t$ is not a stationary point of $L$.
\end{theorem}
As with standard EM, the last fact guarantees that the sequence $\{L(\theta^t)\}$ converges, and, under a curvature condition on the model, that the sequence $\theta^*$ converges to a stationary point as well~\citep[see][]{Dempster1977}.

Sinkhorn EM only differs from the standard algorithm in the E-step.
The original E-step corresponds to the computation of $q^\theta(\cdot|y)$, which for mixtures is the matrix of \textit{responsibilities}, the (posterior) probability of each data point being assigned to a particular component.
It is possible to give a similar meaning to our modified E-step: indeed, it follows from the semi-dual formulation of entropic OT~\citep{Cuturi2018} that $\pi^{(t+1)}$ can be written as
\begin{equation}\label{eq:sinkpost}
d\pi^{(t+1)}(k, y)= \frac{\alpha_k(\theta^{(t)}) q^{\theta^{(t)}}_k(y)}{\sum_j \alpha_j(\theta^{(t)}) q^{\theta^{(t)}}_j(y)},  \quad\text{with}\quad \alpha_k(\theta^{(t)})=\frac{e^{w_k(\theta^{(t)})}P_0(X = k)}{\sum_j e^{w_j(\theta^{(t)})}P_0(X = j)}
    \end{equation}
    for some function $w: \Theta \to \RR^{|\cX|}$.
In words, $\pi$ is indeed a posterior distribution (or equivalently, a matrix of responsibilities) but with respect to a \textit{tilted} prior $\alpha(\theta^{(t)})$ over the weights, where the amount of tilting is controlled by $w$.

Each E-step of the standard EM algorithm takes time $O(n \cdot |\cX|)$, where $n$ is the number of data points.
The E-step of sEM can be implemented via Sinkhorn's algorithm~\citep{Peyre2019}, which converges in $\tilde O(n \cdot |\cX|)$ time~\citep{Altschuler2017}.
As we show in our experiments, this mild overhead in operation complexity is easily compensated for in practice by the fact that sEM typically requires many fewer E-steps to reach a good solution.

\begin{figure}[t]
\includegraphics[width=1.0\textwidth]{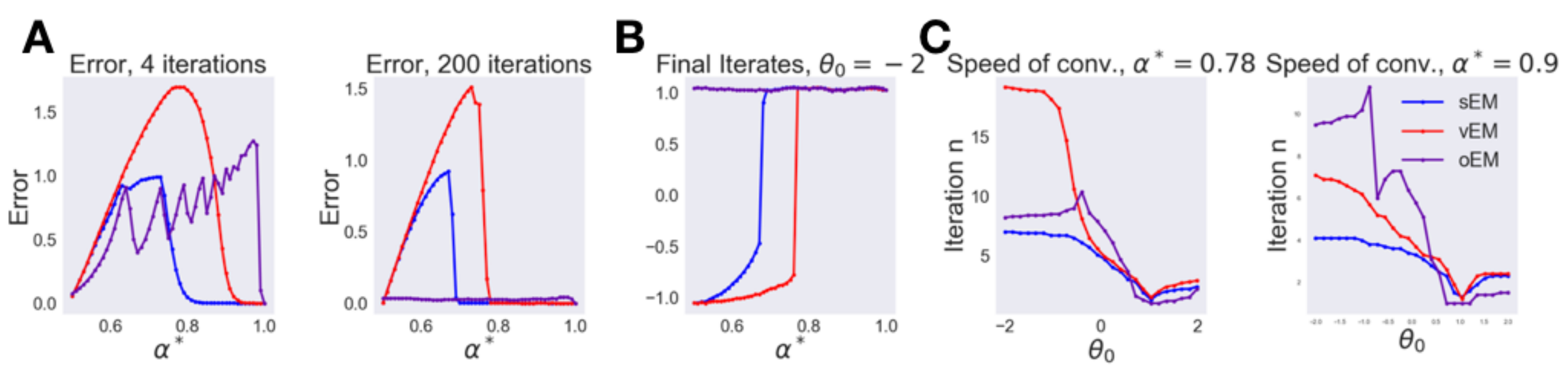}
\caption{Convergence of the example in \eqref{eq:mixtureg}. \textbf{A} The maximum error over all initializations after $n=4$ and $n=200$ iterations, as a function of $\alpha^*$. \textbf{B} Final iterates as a funcion of $\alpha^*$, when the starting point is $\theta_0=-2$. \textbf{C} The number of iterations required to be $\epsilon$-close to the true $\theta^*$, as a function of initialization. The weights $\alpha^*$ are chosen so that all three algorithms converge.}
\label{fig:example1}
\end{figure}
\section{Convergence analysis for mixtures of two Gaussians}
\label{sub:assy}
In this section, we rigorously establish convergence guarantees of sEM for a simple model.
Here, we consider the mixture
\begin{equation}
\label{eq:mixtureg} q^\theta(y) = \alpha^* \mathcal{N}(y;\theta, 1) +(1-\alpha^*)\mathcal{N}(y;-\theta,1)\,,
\end{equation}
where the unknown parameter $\theta$ takes the value $\theta^*>0$ and $\alpha^*$ is known and fixed. In the population limit, we compare the properties of the usual population negative log-likelihood $l_{\alpha^*}$ and the entropic OT loss $L_{\alpha^*}$, and the corresponding EM algorithms that derive from each. By symmetry, we assume that $\alpha^*\geq1/2$ without loss of generality.

This model was studied in great detail by~\citet{Xu2018}, who compared two different procedures: vanilla EM, which is the standard EM algorithm for the model~\eqref{eq:mixtureg}, and ``overparameterized EM,'' which is EM on the larger model where the mixing weight $\alpha$ is also allowed to vary.
\Citet{Xu2018} showed that, for any $\theta^* > 0$, there exist values of $\alpha^*$ for which vanilla EM converges to a spurious fixed point different from $\theta^*$ when initialized at $\theta^0 < - \theta^*$.
This arises because, for these values of $\alpha^*$, the true parameter $\theta^*$ is not the unique fixed point of the log-likelihood $\ell$.
Our first result shows that sEM is less prone to this bad behavior.
\begin{theorem}\label{teo:mixtureloss}
For any $\theta^* > 0$, the set of $\alpha^*$ for which $\theta^*$ is the unique stationary point of $L_{\alpha^*}$ is strictly larger than the one for $l_{\alpha^*}$.
\end{theorem}



The second result concerns the convergence rate of sEM.
We show that as long as sEM is initialized at $\theta^0 > 0$, it enjoys fast convergence to the global optimum, and there is a large range of initializations on which it never performs worse than vanilla EM.

\begin{theorem}\label{teo:mixture}
For the mixture model \eqref{eq:mixtureg}, for each $\theta^*>0$ and initialization $\theta^0 > 0$, the iterates of sEM converge to $\theta^*$ exponentially fast:For each $\theta^* > 0$ and $\theta_0 > 0$, the iterates of sEM converge to $\theta^*$ exponentially fast:

\begin{equation}\label{emgeom_correct}
|\theta^{t} - \theta^*| \leq \rho^t |\theta^0 - \theta^*|\text{, with $\rho = \exp\left(-\frac{\min\{\theta^0, \theta^*\}^2}{2}\right).$}
\end{equation}
Moreover, there is a $\theta_{\text{fast}} \in(0,\theta^*)$ depending only on $\theta^*$ and $\alpha^*$ such that, if sEM and vanilla EM are both initialized at $\theta^0 \in [\theta_{\text{fast}}, \infty)$, then
\begin{equation}
\label{eq:fast}
|\theta^t - \theta^*| \leq |\theta^t_{\text{vEM}} - \theta^*| \quad \quad \forall t \geq 0\,,
\end{equation}
where $\theta^t_{\text{vEM}}$ are the iterates of vanilla EM.
In other words, when initialized in this region, sEM never underperforms vanilla EM.
\end{theorem}




\section{Empirical results on simulated data}\label{sec:simulation}
In this section we compare the performance of vanilla EM, overaparameterized EM \citep{Xu2018} and Sinkhorn EM on several simulated Gaussian mixtures, summarized as follows. We refer the reader to the appendix for experimental details.
In this section, we measure convergence speed by the number of E-steps each algorithm requires.
In the following section, we compare the actual execution times of each algorithm.
\paragraph{Symmetric mixture of two Gaussians with asymmetric weights}
In Fig.~\ref{fig:example1}, we plot results on the model~\eqref{eq:mixtureg}. Consistent with the results of \citet{Xu2018}, vanilla EM only converges for some values of $\alpha^*$. We show sEM escapes local optima significantly more often than vanilla EM, and typically convergences faster. Although overparameterized EM always escapes local optima, convergence can be slow. When all three algorithms converge, sEM usually converges fastest.

\begin{figure}
\includegraphics[width=1.0\textwidth]{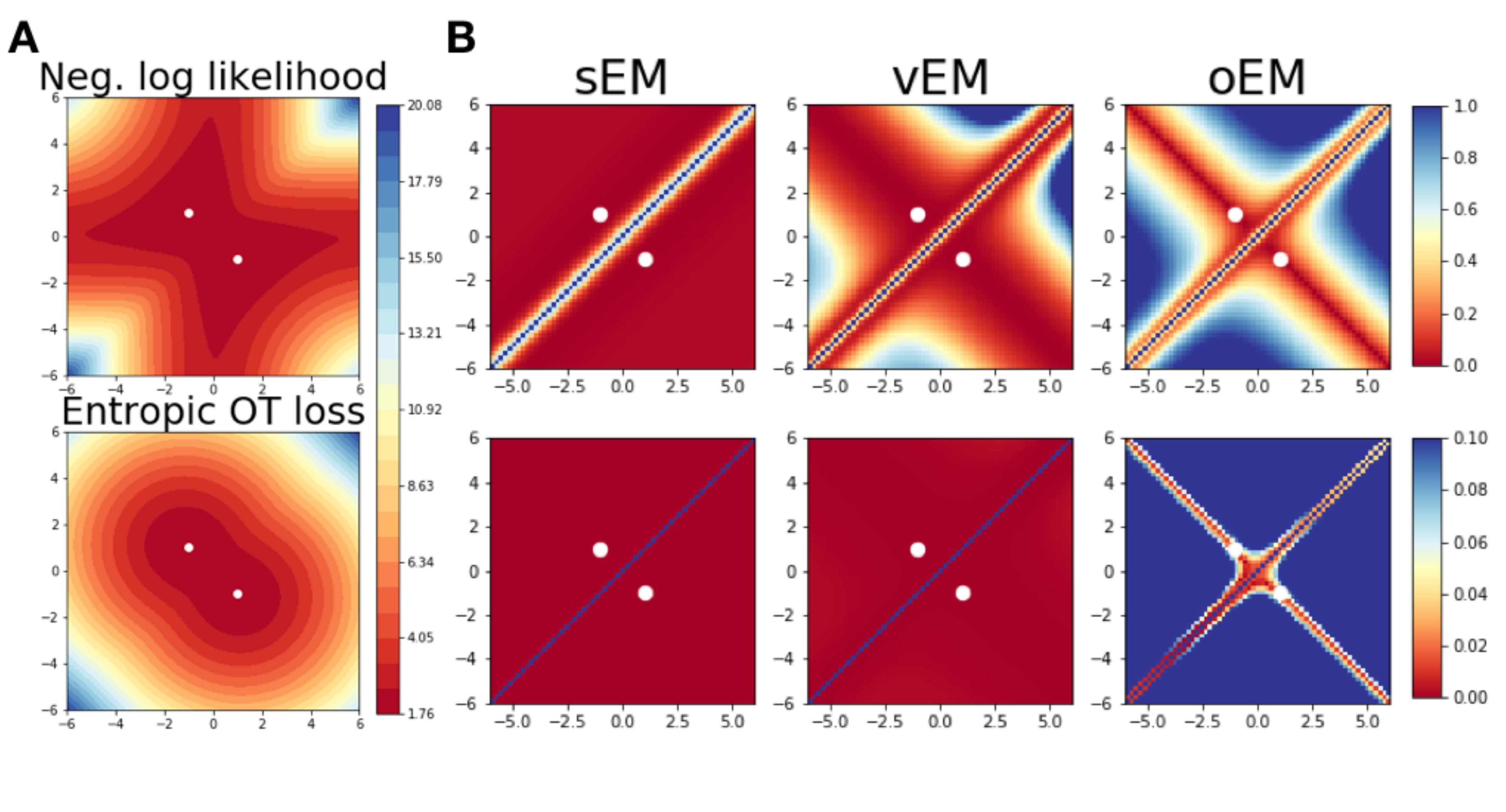}
\caption{Setup for example \eqref{eq:mixtutwo}. \textbf{A} Negative log likelihood and entropic OT loss. Notice that the log-likelihood has saddle points at $(0,\infty),(0,-\infty),(\infty,0)$, and $(-\infty,0)$, while the only stationary points of entropic OT loss are global optima (white dots). \textbf{B} Speed of convergence analysis:  the first row shows the corresponding errors at the first iteration from each initial $\theta^0$. The second row shows the same, after five iterations.}
\label{fig:example2}
\end{figure}

\paragraph{Equal mixture of two Gaussians}
In Fig.~\ref{fig:example2},  we study the model
\begin{equation}\label{eq:mixtutwo} q^{\theta}(y) = \frac 12  \mathcal{N}(y;\theta_1, 1) +\frac 12 \mathcal{N}(y;\theta_2,1). \end{equation}
We assume the true parameters are $(\theta^*_1,\theta^*_2)=(-1,1)$ and study convergence for different initializations of ($\theta^0_1,\theta^0_2)$. All methods converge except when initialized on the line $x=y$, and again, convergence requires the fewest iterations for sEM, whose iterates always fall on the line $x=-y$. In the appendix we show additional experiments with comprehensive choices of parameters, and show that in some modifications of this model (e.g., when the variances have to be estimated) overparameterized EM fails to recover the true parameters.

\paragraph{Mixture of three Gaussians}
In Fig.~\ref{fig:example3a}, we study the model
\begin{equation}\label{eq:mixtuthree} q^{(\theta_1,\theta_2,\theta_3)}(y) = \frac{1}{3} \mathcal{N}(y;\theta_1, 1) +\frac{1}{3}\mathcal{N}(y;\theta_2,1)+\frac{1}{3}\mathcal{N}(y;\theta_3,1), \end{equation}
and assume that the true parameters satisfy $(\theta^*_1,\theta_2^*,\theta^*_3)=(-\mu,0,\mu)$ for some $\mu > 0$.
We run our experiments on a dataset consisting of $1000$ samples from the true distribution.
Our results show that when $\mu$ is small, overaparameterized EM may overfit and converge to poor solutions. We explore this phenomenon in more detail in the appendix.

\begin{figure}[ht]
\vspace{-0.2cm}
\includegraphics[width=1.0\textwidth]{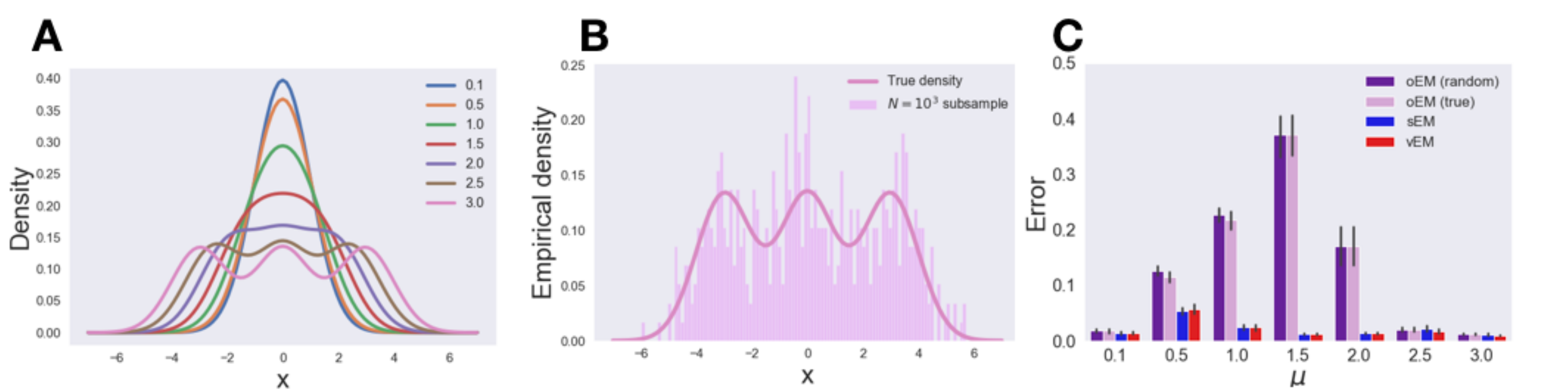}
\caption{Overparameterized EM has the worst performance on the mixture of three Gaussians experiment. \textbf{A}: Densitites of all considered mixtures for different values of $\mu$. \textbf{B}: Density when $\mu=3$ along with a sampled dataset. \textbf{C}: Estimation error.}
\label{fig:example3a}
\end{figure}


\section{Application to inference of neurons in \textit{C. elegans}}\label{sec:c_elegans}
\vspace{-0.2cm}
\begin{figure}[ht]
\includegraphics[width=1.0\textwidth]{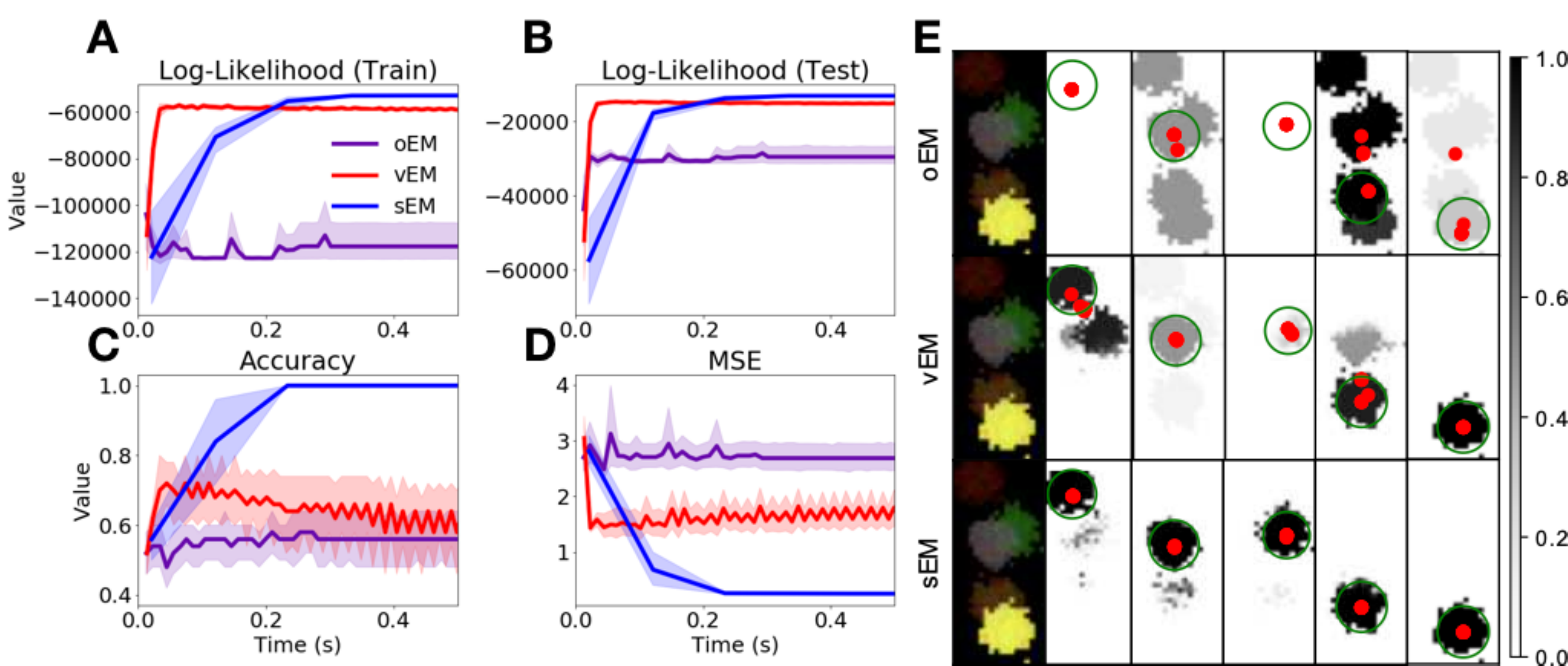}
\caption{Performance evaluation of Sinkhorm EM, vanilla EM and overparameterized EM on \textit{C. elegans} neuron identification and segmentation task (see the appendix for a detailed definition of our metrics). \textbf{A-D}. Training (\textbf{A}) and test (\textbf{B}) log-likelihoods, segmentation accuracy (\textbf{C}) and mean squared error (\textbf{D}) for the three methods. \textbf{E}. The visual segmentation quality. Each row denotes a different method; the first column shows the observed neuronal pixel values (identical for all three methods), and the remaining columns indicate the mean identified segmentation of each neuron (in grayscale heatmaps) and the inferred cell center in red dots, over multiple randomized runs. The ground truth neuron shape is overlaid in green.}
\label{fig:worm_fig}
\end{figure}

Automated neuron identification and segmentation of \textit{C. elegans} is crucial for conducting high-throughput experiments for many applications including the analysis of gene expression profiles, cell fate studies \citep{sulston1983embryonic}, stem cell research, and the study of circuit-level neuronal dynamics \citep{kato2015global}.
A recently introduced novel transgenic strain of \textit{C. elegans} has a deterministic coloring of each neuron, which enables the disambiguation of nearby neurons and aid in their identification. The coloring scheme and the stereotypical positioning of the cells in populations of \textit{C. elegans} has allowed the construction of a statistical atlas that encodes the canonical neuron positions and their colors~\citep{yemini2019neuropal}.

This neural statistical atlas of \textit{C. elegans} provides us with a strong prior to guide identification and segmentation of neurons. We model the assignment of pixels to neurons through a \textit{Bayesian} Gaussian mixture model (GMM) where the model parameters $\mu_k, \Sigma_k$ correspond to cell centers (and colors) and shapes respectively. The responsibilities matrix $\pi$ which encodes the probabilistic assignment of the pixels to the cells can be considered as a probabilistic segmentation of the image to regions with high probabilities for each cell.

Starting from a subset of cells in the neural statistical atlas of \textit{C. elegans} neurons, we first sampled neuron locations and colors $\mu_k$ given their \textit{prior} canonical locations, colors and their variance. Then, we sampled the pixels from a GMM with previously sampled cell centers and colors (and independently sampled cell shapes) as its parameters. We then aim to recover centers and shapes using the MAP estimate in the following statistical model: 
\begin{gather}
P(Y,\mu,\Sigma) =\prod_{i=1}^n \left(\sum_{k=1}^{K}\alpha_k \mathcal{N}(Y_i|\mu_k, \Sigma_k)\right) \prod_{k=1}^K \mathcal{N}(\mu_k|\mu_k^a,\Sigma_k^a),
\end{gather}

where each observation $Y_i \in \mathbb{R}^6$ is the concatenation of the pixel location $l_i \in \mathbb{R}^3$ and pixel color $c_i \in \mathbb{R}^3$. Also $\mu_k \in \mathbb{R}^6$ and $\Sigma_k \in \mathbb{R}^{6 \times 6}$ are mean and covariance parameters of the GMM, which, in turn, depend on their priors $\mu_k^a,\Sigma_k^a$.

We estimate the model parameters with vanilla EM, overparameterized EM, and Sinkhorn EM. The cluster centers are initialized randomly while the covariances are set to be constant and allowed to be updated for all three methods (other initialization and update configurations are presented in the appendix). Notice that although this deviates from the fully frequentist framework, any version of the EM algorithm can still be applied by incorporating the prior information in the M step~\citep{Ormoneit1996}.


Each of the methods returns inferred cell centers, colors, and shapes, as well as a $\pi$ matrix that can be used for probabilistic segmentation of the images. Fig. \ref{fig:worm_fig}(\textbf{A-D}) shows the evaluation of the three algorithms. sEM outperforms vEM and oEM in all four metrics, and in comparable time. The plots also demonstrate that vEM shows oscillatory behavior which sEM avoids, and that oEM has a tendency to capture the wrong components, leading to a lower accuracy and higher MSE.

To qualitatively evaluate the convergence properties of each of the algorithms for random initialization, we ran each method 10 times and computed the probabilistic segmentation maps for each component. Fig. \ref{fig:worm_fig}-\textbf{E} shows the average segmentation maps for 20 runs as well as the inferred $\mu_k$ values. The segmentation for sEM is crisper than other methods, with clear edges and boundaries, and the centers end up very close to the true value in almost every iteration.

\section{Broader impacts}
The success of sEM has implications for improving research in neuroscience and other biological sciences, where clustering and identification techniques are widely used.
The field of machine learning often has unintended consequences across many domains, and applying sEM to practical problems requires careful analysis of societal risks and benefits.

\section{Acknowledgements}
Gonzalo Mena is funded by a Harvard Data Science Initiative Fellowship. Amin Nejatbakhsh and Erdem Varol are funded by NSF NeuroNex Award DBI-1707398 and The Gatsby Charitable Foundation. Jonathan Niles-Weed acknowledges the support of the Institute for Advanced Study. Authors thank Ji Xu for valuable comments and Eviatar Yemini for providing the C. elegans dataset.

\pagebreak
\appendix

\renewcommand\thefigure{A\arabic{figure}} 
\renewcommand\thelemma{A\arabic{lemma}} 
\renewcommand\thetheorem{A\arabic{theorem}} 
\renewcommand\theproposition{A\arabic{proposition}} 
\renewcommand{\theequation}{A\arabic{equation}}
\setcounter{figure}{0}
\setcounter{equation}{0}
\setcounter{theorem}{0}
\setcounter{proposition}{0}
\setcounter{lemma}{0}

\section{Omitted Proofs}
\subsection{Proof of Proposition~\ref{mini}}
First, we prove~\eqref{eq:dominate}.
By the disintegration theorem we can express each coupling $\pi\in \Pi(P_0, \mu_Y)$ as $\dd \pi(y,x) = \dd \mu_Y(y) \dd P(x|y)$ where  $P$ is a kernel of conditional probabilities satisfying $\dd P_0(x) = \int \dd P(x|y) \dd \mu_Y(y)$. Denote $\bar{\Pi}(P_0, \mu_Y)$ the set of such kernels. Then, we have:
\begin{align*}
S_\theta(P_0, \mu_Y) & = \min_{P\in \bar{\Pi}(P_0, \mu_Y)} \E_{Y \sim \mu_Y} \left(\E_{X \sim P(\cdot | Y)} g^\theta(X, Y) + H(P(\cdot | Y) | P_0)\right) \\
& = \min_{P\in \bar{\Pi}(P_0, \mu_Y)} \E_{Y \sim \mu_Y} \left(\E_{X \sim P(\cdot | Y)} g^\theta(X, Y) + H(P(\cdot | Y) | P_0)\right) \\
& = \min_{P \in \bar{\Pi}(P_0, \mu_Y)}\E_{Y \sim \mu_Y}\left(-\E_{X\sim P(\cdot|Y)}\left(\log \frac{\dd Q^\theta_{X,Y}(X, Y)}{\dd P_0(X)}\right)+E_{X\sim P(\cdot|Y)}\left(\log \frac{\dd P(X|Y)}{\dd P_0(X)}\right)\right) \\
& = - \max_{P \in \bar{\Pi}(P_0, \mu_Y)} \E_{Y \sim \mu_Y} F_Y(P(\cdot |Y), \theta)\,,
\end{align*}

where $F_Y$ was defined in \eqref{eq:f_def}. \Citet{Neal1998} show that for each fixed $\theta$ and $Y$, the functional $F_Y$ is maximized at the measure $\tilde P$ defined by $\dd \tilde P(x) = Q^{\theta}_{X|Y}(x | Y)$, i.e., the conditional distribution of $X$ given $Y$ under $Q_Y^\theta$, and that
\begin{equation*}
F_Y(Q^{\theta}(\cdot | Y), \theta) = \log q_Y^\theta(Y)\,.
\end{equation*}
Therefore, by dropping the constraint $P \in \bar{\Pi}(P_0, \mu_Y)$, we obtain
\begin{equation*}
S_\theta(P_0, \mu_Y) = - \max_{P \in \bar{\Pi}(P_0, \mu_Y)} \E_{Y \sim \mu_Y} F_Y(P(\cdot |Y), \theta) \geq - \max_{P} \E_{Y \sim \mu_Y} F_Y(P(\cdot |Y), \theta) = - \E_{Y \sim \mu_Y} \log q^\theta_Y(Y)\,,
\end{equation*}
as desired.
The inequalities in~\eqref{eq:ell_dominate} then follow upon choosing $\mu_Y = Q^{\theta^*}_Y$ and $\mu_Y = \frac 1n \sum_{i=1}^n \delta_{Y_i}$, respectively.

Finally, if $(X, Y) \sim Q_{X,Y}^{\theta^*}$, then by definition the kernel $y \mapsto Q_{X|Y}^{\theta^*}(\cdot | y)$ lies in $\bar{\Pi}(P_0, Q^{\theta^*}_Y)$.
We obtain
\begin{equation*}
L(\theta^*) = - \max_{P \in \bar{\Pi}(P_0, Q^{\theta^*}_Y)} \E_{Y \sim Q^{\theta^*}_Y} F(P(\cdot |Y), \theta^*) \leq - \E_{Y \sim Q^{\theta^*}_Y} F(Q^{\theta^*}_{X|Y}(\cdot | Y), \theta^*) = \ell(\theta^*)\,,
\end{equation*}
so in fact $L(\theta^*) = \ell(\theta^*)$.
Since $\theta^*$ minimizes $\ell$, it therefore must also minimize $L$.
\qed

\subsection{Proof of Proposition \ref{prop:curvature}}
\begin{proof}
We start with the following semi-dual formulation ~\citep{Cuturi2018} for $L(\theta)$:
\begin{equation}
\label{eq:semidual1}
L(\theta) = \max_{w\in{\mathbb{R}^{|\cX|}}} \left[\sum_{i=1}^{|\cX|} w_k\alpha_k -\int \log \left(\sum_{k=1}^{|\cX|} \alpha_k \exp\left(w_k + \log q^{\theta}_k(y) \right)\right)\dd \mu_Y(y)\right].
\end{equation}
Notice the maximum above is realized for many $w$, as one may add an arbitrary constant to any coordinate of $w$ without changing the right hand side. Therefore, we can assume $w(|\cX|)=0$ and for convenience we define, for $w\in\mathbb{R}^{|\cX|-1}$, the function 
\begin{equation}\label{eq:l2}
L_2(\theta,w)= \sum_{i=1}^{|\cX|} w_k\alpha_k -\int \log \left(\sum_{k=1}^{|\cX|} \alpha_k \exp\left(w_k + \log q^{\theta}_k(y) \right)\right)d\mu_Y(y).\end{equation}
and note that \begin{equation}\label{eq:l2q} L_2(\theta,0)= -E_{\mu_Y}(\log q^\theta(Y)).\end{equation} Now,
let's call $w_\theta$ the one that achieves the maximum in \eqref{eq:semidual1}. We will follow an envelope-theorem like argument: we have $L(\theta)=L_2(\theta,w_\theta)$, and based on this we may compute first and second derivatives using the chain rule 
\begin{equation}\label{eq:partialL}
\frac{\partial L}{\partial \theta}(\theta) = \frac{\partial L_2}{\partial \theta}(\theta, w_\theta) + \frac{\partial L_2}{\partial w}(\theta, w_\theta) \frac{\partial w_\theta}{\partial \theta},
\end{equation}

and 

\begin{eqnarray}
\label{eq:sec}
\frac{\partial^2 L}{\partial \theta^2}(\theta) &=& \frac{\partial^2 L_2}{\partial \theta^2}(\theta, w_\theta)  + \frac{\partial^2 L_2}{\partial w \partial \theta}(\theta, w_\theta) \frac{\partial w_\theta}{\partial \theta}
 +  \frac{\partial^2 L_2}{\partial \theta \partial w}(\theta, w_\theta) \frac{\partial w_\theta}{\partial \theta} \\
 \nonumber & &
 +\frac{\partial^2 L_2}{\partial w^2}(\theta, w_\theta) \left(\frac{\partial w_\theta}{\partial \theta}\right)^2+ \frac{\partial L_2}{\partial w}(\theta, w_\theta) \frac{\partial w_\theta}{\partial \theta}.
\end{eqnarray}
 But by optimality of $w_\theta$, for every $\theta$ we have
 \begin{equation}\label{eq:sec2}0 =  \frac{\partial L_2}{\partial w}(\theta, w_\theta) \text{ and} \quad 0 =  \frac{\partial^2 L_2}{\partial \theta \partial w}(\theta, w_\theta)+ \frac{\partial^2 L_2}{ \partial w^2}(\theta, w_\theta)\frac{\partial w_\theta}{\partial \theta}.\end{equation}

Therefore, by combining \eqref{eq:sec} and \eqref{eq:sec2} we obtain 
\begin{equation}\label{eq:sec3}
\frac{\partial^2 L}{\partial \theta^2}(\theta) = \frac{\partial^2 L_2}{\partial \theta^2}(\theta, w_\theta)  - \frac{\partial^2 L_2}{\partial w \partial \theta}(\theta, w_\theta)\left(\frac{\partial^2 L_2}{ \partial w^2}  (\theta, w_\theta)
\right)^{-1} \frac{\partial^2 L_2}{\partial \theta \partial w}(\theta, w_\theta).
\end{equation}
Additionally, it is easy to see that since $\mu_Y=Q^{\theta^*}$ (population limit), $w_{\theta^\ast}=0$, and that as with \eqref{eq:l2q} it also holds that for each $\theta$
\begin{equation}\label{eq:l2q} \frac{\partial^2 L_2}{\partial \theta^2}(\theta, 0)= -\frac{\partial^2 }{\partial \theta^2}E_{\mu_Y}(\log q^\theta(Y)).\end{equation}
Therefore, 
\begin{equation}\label{eq:sec4}
\frac{\partial^2 L}{\partial \theta^2}(\theta^*) =  -\frac{\partial^2 }{\partial \theta^2}E_{\mu_Y}(\log q^{\theta^\ast}(Y)) - \frac{\partial^2 L_2}{\partial w \partial \theta}(\theta, 0)\left(\frac{\partial^2 L_2}{ \partial w^2}  (\theta^*, 0)
\right)^{-1} \frac{\partial^2 L_2}{\partial \theta^\ast \partial w}(\theta^*, 0).
\end{equation}
To conclude the proof, then, it suffices to show that $\frac{\partial^2 L_2}{ \partial w^2}  (\theta^*, 0)$ is negative definite. To see this, we write
\begin{equation} \frac{\partial^2 L_2}{ \partial w^2}  (\theta, w)= -\int\left(Diag(v(y,\theta,w))-v(y,\theta,w)v(y,\theta,w)^\top\right) \dd \mu_Y(y), \end{equation}
where  for $k=1,\ldots, {|\cX|}-1$ \begin{equation}
v(y,\theta,w)_k= \frac{\alpha_k \exp\left(w_k + \log q_k^\theta(y)\right)}{\sum_{k=1}^{|\cX|}\alpha_k \exp\left(w_k + \log q_k^\theta(y)\right)}.
\end{equation}
Define the (symmetric) matrix $I_{k,k'}$ for $k,k'=1,\ldots |\cX|$ as $$ I_{k,k'}:=\int\left( \frac{\alpha_k q^{\theta^\ast}_k(y)}{\sum_{k=1}^{|\cX|} \alpha_k q^{\theta^\ast}_k(y))}   \frac{\alpha_{k'} q^{\theta^\ast}_{k'}(y) }{\sum_{k'=1}^|\cX| \alpha_{k'} q^{\theta^\ast}_{k'}(y)}   \right) \dd \mu_Y(y).$$   
if $k\neq k$ and otherwise  $$ I_{k,k}:=-\int\left( \frac{\alpha_k q^{\theta^\ast}_k(y)}{\sum_{k=1}^{|\cX|} \alpha_k q^{\theta^\ast}_k(y)}   - \frac{\alpha_k q^{\theta^\ast}_k(y)}{\sum_{k=1}^{|\cX|} \alpha_k q^{\theta^\ast}_k(y)}   \frac{\alpha_k q^{\theta^\ast}_k(y) }{\sum_{k=1}^{|\cX|} \alpha_k q^{\theta^\ast}_k(y)}   \right) \dd \mu_Y(y).$$
Notice this matrix coincides with $\frac{\partial^2 L_2}{ \partial w^2}  (\theta^\ast, 0)_{k,k'}$ for $k\leq |\cX|-1$. 
Since $\mu_Y(y) =\sum_{k=1}^{|\cX|} \alpha_k    q^{\theta^\ast}_k(y) $ we have that $$ \sum_{k'=1, k'\neq k}^{|\cX|} I_{k,k'} = \alpha_k - \int\left( - \frac{\alpha_k q^{\theta^\ast}_k(y)}{\sum_{k=1}^{|\cX|} \alpha_k q^{\theta^\ast}_k(y)}   \frac{\alpha_k q^{\theta^\ast}_k(y) }{\sum_{k=1}^{|\cX|} \alpha_k q^{\theta^\ast}_k(y)}   \right) \dd \mu_Y(y)  =-I_{k,k}>0 $$
Then, $I$ is a negative weighted Laplacian matrix and
$$x^\top I x = \frac{1}{2}\sum^{|\cX|}_{k,k'} I_{k,k}(x_k-x_{k'})^2\leq 0.$$
The above expression is zero only if $x$ is a constant vector. Since $\frac{\partial^2 L_2}{ \partial w^2}  (\theta^\ast, 0)$ is a submatrix of $I$, it is also negative semidefinite. Now, suppose
$z^\top \frac{\partial^2 L_2}{ \partial w^2}  (\theta^\ast, 0) z=0$, then, if $x_k=z_k$ for $k\leq |\cX|-1$ and $x_{|\cX|}=0$ we have  $z^\top \frac{\partial^2 L_2}{ \partial w^2}  (\theta^\ast, 0) z=x^\top I x =0$ and since $x$ must be constant, $z=0$. Therefore, $\frac{\partial^2 L_2}{ \partial w^2}(\theta^\ast, 0)$ is negative definite and the proof is concluded.
\end{proof}
\subsection{Proof of Theorem \ref{teo:em}}
\begin{proof}

The proof is an adaptation of the original method for the EM algorithm introduced in \cite{Wu1983}.  By definition of $\theta^t,\theta^{t+1}$ the following inequality holds
\begin{equation}
\label{ineqem}
 -E_{(X,Y)\sim \pi^{t+1}}\left(\log q^{\theta^{t+1}}(X,Y)\right) +H(\pi^{t+1}|P_0 \otimes \mu_Y) \leq  -E_{(X,Y)\sim \pi^{t+1}}\left(\log q^{\theta^{t}}(X,Y)\right) +H(\pi^{t+1}|P_0 \otimes \mu_Y).\end{equation}
By the definition of $L(\cdot)$, the right hand side equals $L(\theta^t)$. Also, since $P_0$ is fixed the feasible set of couplings $\pi$ is the same for every $\theta$. In particular, $\pi^{t+1}$ is a feasible coupling for the problem defining $L(\theta^{t+1})$, and by virtue of \eqref{ineqem} we conclude $L(\theta^{t+1})\leq L(\theta^t)$.

We now show the inequality is strict if $\theta^t$ is not a stationary point. Define $$\tilde{L}(\theta,\pi)=  -E_{(X,Y)\sim \pi}\left(\log q^\theta(X,Y)\right) +H(\pi|P_0 \otimes \mu_Y).$$
Therefore, $L(\theta)=\tilde{L}(\theta,\pi_\theta)$ where $\pi_\theta$ minimizes $\tilde{L}(\theta,\pi)$ over $\pi$, for a fixed $\theta$. By the chain rule and optimality of $\pi_\theta$ we have
\begin{equation}\label{eq:sta}\frac{\partial L}{\partial \theta}(\theta)=\frac{\partial \tilde{L}}{\partial \theta} (\theta,\pi_\theta) + \frac{\partial \tilde{L}}{\partial \pi} (\theta,\pi_\theta) \frac{\partial \pi_\theta}{\partial \theta}=\frac{\partial \tilde{L}}{\partial \theta} (\theta,\pi_\theta).\end{equation}

Since $\theta^t$ is not a stationary point, the above implies that $\frac{\partial \tilde{L}}{\partial \theta} (\theta^t,\pi_{\theta^t})\neq 0$. Therefore, $\theta^t$ is not a stationary point for the function that is optimized at the $M$-step. In consequence, this $M$-step strictly decreases this function, and hence, of $L(\cdot)$, by definition.
\end{proof}

\subsection{Proof of Theorems \ref{teo:mixtureloss} and \ref{teo:mixture}}
The proof of Theorem \ref{teo:mixtureloss} relies on an analysis of the functions $L_{\alpha^*}$ and $\ell_{\alpha^*}$ and their derivatives.
Fig.~\ref{fig:1} depicts the main properties of the functions that will be used in the proofs.
The first row shows $L_{\alpha^*}(\theta)\geq \ell_{\alpha^*}(\theta)$, which is the conclusion of Proposition \ref{mini}.
The second through fourth rows illustrate the behavior of the derivatives $L'$ and $\ell'$.
We show in Proposition \ref{prop:deri} that $L'_{\alpha^*}(\theta)\geq \ell'_{\alpha^*}(\theta)$ for all $\theta < 0$, which is clearly visible in the second and fourth row. In the third row, we plot the absolute values of the derivatives, with stationary points visible as cusps.
In the last row, we plot an important auxiliary function, which is described in more detail below.

As mentioned in the main text, we assume $\alpha^*>0.5$, by a simple symmetry argument. The fourth column in Fig.~\ref{fig:1} illustrates this symmetry. Additionally, we exclude the $\alpha^*=0.5$ from our analyses, as in this case the entropic OT loss coincides with the negative log likelihood (last column of Fig.~\ref{fig:1}) and sEM and vEM define the same algorithm.

For the proof of Theorem~\ref{teo:mixture}, we make several additional definitions.
We recall the semi-dual formulation~\eqref{eq:semidual1}.
The first-order optimality conditions for $w$ read
\begin{equation}\label{eq:optimalpha} \alpha^*=\int \frac{e^{w_1}\alpha^*e^{-(\theta-y)^2/2}}{e^{w_1}\alpha^*e^{-(\theta-y)^2/2}+e^{w_2}(1-\alpha^*)e^{-(\theta+y)^2/2}}q^{\theta^*}(y)\dd y.\end{equation}
The above condition can be expressed in terms of the \textit{tilted} $\alpha(\theta)$ introduced in \eqref{eq:sinkpost}. This $\alpha(\theta)$ the unique number in $[0,1]$ satisfying
\begin{equation}\label{eq:optimalpha2} \alpha^*=G(\theta,\alpha(\theta)),\end{equation} 
with $G(\theta,\alpha)$ defined as
\begin{align}\label{eq:optimalpha3} G(\theta,\alpha)&:=\int \frac{\alpha e^{-(\theta-y)^2/2}}{\alpha e^{-(\theta-y)^2/2}+(1-\alpha)e^{-(\theta+y)^2/2}}q^{\theta^*}(y)\dd y = \int \frac{\alpha e^{\theta y}}{\alpha e^{\theta y }+(1-\alpha)e^{-\theta y}}q^{\theta^*}(y)\dd y. \end{align}

We plot the tilting $\alpha(\theta^*)$ in the last row of Fig.~\ref{fig:1}.

To analyze the behavior of Sinkhorn EM and vanilla EM, we also introduce the auxiliary function $F(\theta, \alpha)$ defined by
\begin{equation}\label{eq:def_f}
F(\theta, \alpha) := \int_{\R} y \frac{\alpha e^{\theta y} - (1-\alpha) e^{-\theta y}}{\alpha e^{\theta y} + (1-\alpha) e^{-\theta y}} q^{\theta^*}(y) \dd y\,.
\end{equation}

With this notation, the updates of sEM satisfy
\begin{equation*}
\theta^{t+1}_{sEM} = F(\theta^t_{sEM}, \alpha(\theta^t_{sEM}))\,,
\end{equation*}
where $\alpha(\theta)$ is defined in~\eqref{eq:optimalpha2}.
On the other hand, the updates of vEM satisfy
\begin{equation*}
\theta^{t+1}_{vEM} = F(\theta^t_{vEM}, \alpha^*).
\end{equation*}

\begin{proof}[Proof of Theorem \ref{teo:mixtureloss}]
We assume as above that $\alpha^* > 0.5$.
First, we show that $L_{\alpha^*}$ never has spurious stationary points on $(0, \infty)$.
This follows from Theorem~\ref{teo:mixture}.
Indeed, that theorem guarantees global convergence of sEM for all positive initializations.
As in the proof of Theorem~\ref{teo:em}, a stationary point of $L(\theta)$ is also a stationary point of the function $\theta \mapsto \E_{Y \sim \mathrm Q^{\theta^*}} F_Y(P(\cdot | Y), \theta)$, where $P$ is the optimal kernel for the parameter $\theta$.
However, this function is convex in $\theta$ for any choice of $P$, so if $\theta'$ is a stationary point of $L$, then $\theta'$ minimizes $\theta \mapsto \E_{Y \sim \mathrm Q^{\theta^*}} F_Y(P(\cdot | Y), \theta)$, which implies that $\theta'$ is also a fixed point of the dynamics of sEM.
Since Theorem~\ref{teo:mixture} guarantees that sEM converges to $\theta^*$ for any positive initialization, this implies that there are no spurious stationary points on $(0, \infty)$.

We now show that if $L_{\alpha^*}$ has a spurious stationary point, then so does $\ell_{\alpha^*}$.
Suppose that $L_{\alpha^*}$ has a stationary point $\theta \in (-\infty, 0]$.
In Proposition \ref{prop:deri} we show that if $\theta \leq 0$, then $L'_{\alpha^*}(\theta)>\ell'_{\alpha^*}(\theta)$.
Therefore, if $\theta$ is stationary point of $L_{\alpha^*}$, then $\ell'_{\alpha^*}(\theta) < 0$.
Since $\ell_{\alpha^*}$ is continuously differentiable and $\ell'_{\alpha^*}(0) = (2{\alpha^*}^2 - 1)^2 > 0$, there must be a $\theta' \in (\theta, 0)$ such that $\ell'_{\alpha^*}(\theta')  = 0$.
Therefore $\ell_{\alpha^*}$ also has a spurious stationary point.

Finally, to show that the set of $\alpha^*$ for which $\ell_{\alpha^*}$ has a spurious stationary point is strictly larger than the corresponding set for $L_{\alpha^*}$, we note that the arguments in the proof of Theorem 1 and Lemma 4 in~\citep{Xu2018} establish that there is $\delta > 0$ such that if $\alpha^* = 0.5 + \delta$ then $\ell_{\alpha^*}$ has a single spurious stationary point on $(-\infty, 0)$, and if $\alpha^* > 0.5 + \delta$, then $\ell_{\alpha^*}$ does not have any spurious stationary points. Sine $\ell'_{\alpha^*}(\theta)$ is a continuous function of $\alpha^*$, this implies that $\ell'_{0.5 + \delta}$ is nonnegative for all $\theta < 0$.
Since $L'_{0.5 + \delta}(\theta)>\ell'_{0.5 + \delta}(\theta)$ for all $\theta < 0$, we obtain that $L'_{0.5 + \delta}$ has no spurious stationary points.
\end{proof}

\begin{figure}[ht!]
\hspace{-1.5cm}
\includegraphics[width=1.2\textwidth]{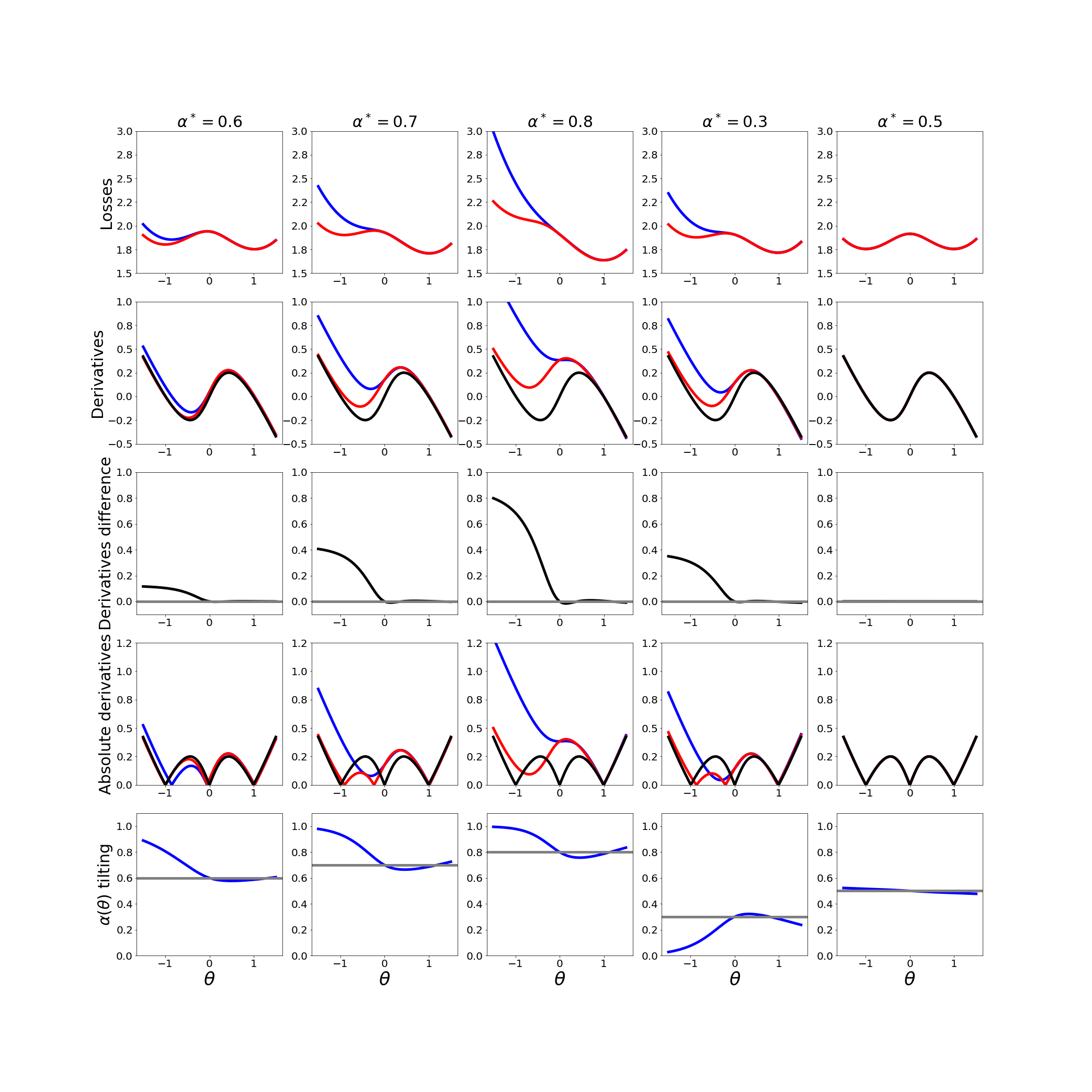}.
\caption{Behavior of $L$, $\ell$ and their derivatives for different values of $\alpha^*$.  Black lines correspond to the reference $\alpha=0.5$ (also in last column). \textbf{First row} entropic OT ($L$, blue) and negative log likelihood $\ell$ (red). \textbf{Second row} derivatives of $L$ and $\ell$. \textbf{Third row} difference between the derivatives $L$ and $\ell$. \textbf{Fourth row} absolute value of the derivatives. \textbf{Fifth row} optimal $\alpha(\theta)$ from the semi-dual entropic OT formulation. }
\label{fig:1}
\end{figure}

\begin{proof}[Proof of Theorem \ref{teo:mixture}, equation \eqref{emgeom_correct}]
Let us fix $\alpha^* > 0.5$.
We first recall the results of \cite[Theorem 1]{DasTzaZam17}, where the bound \eqref{emgeom_correct} is stated for the vanilla EM algorithm in the symmetric mixture ($\alpha^*=0.5$). Let us denote $\theta^{t}_{EM_0}$ for the iterates of vanilla EM on the symmetric mixture, initialized at $\theta^0 > 0$.
We write $\theta^t_{sEM}$ for the iterates of sEM on the \emph{asymmetric} mixture.
We will show that, for all $t \geq 0$, $\theta^t_{sEM}$ and $\theta^{t}_{EM_0}$ satisfy
\begin{align}
\theta^* \leq \theta^t_{sEM} & \leq \theta^t_{EM_0} \quad \text{if $\theta^0 \geq \theta^*$,} \label{monotone>theta}\\
\theta^* \geq \theta^t_{sEM} & \geq \theta^t_{EM_0} \quad \text{if $0 < \theta^0 \leq \theta^*$.} \label{monotone<theta}
\end{align}
This will then prove the claim, since it implies
\begin{equation*}
|\theta_{sEM}^{t} - \theta^*| \leq |\theta_{EM_0}^{t} - \theta^*| \leq \rho^{t} |\theta_{0} - \theta^*|.
\end{equation*}

It remains to prove~\eqref{monotone>theta} and~\eqref{monotone<theta}.
Recall the function $F$ defined in~\eqref{eq:def_f}.
We first show that 
\begin{equation}\label{eq:fcases}
F(\theta,\alpha(\theta)) \begin{cases} \leq \theta^*, &
  0< \theta< \theta^*  \\ = \theta^*,  &
  \theta =\theta^* \\ \geq \theta^*, &
\theta>\theta^*  \end{cases}.\end{equation}
This implies the first inequalities of ~\eqref{monotone>theta} and~\eqref{monotone<theta}.
To show \eqref{eq:fcases}, notice first that clearly $F(\theta^*,\alpha(\theta^*))=F(\theta^*,\alpha^*)=\theta^*$. 
It is therefore enough to establish that $\theta \mapsto F(\theta,\alpha(\theta))$ is non-decreasing.
Let us define $f(\theta)=F(\theta,\alpha(\theta))$. We then have
 \begin{equation} \label{eq:partialf}f'(\theta)=\frac{\partial F}{\partial \theta}(\theta,\alpha(\theta))+\frac{\partial F}{\partial \alpha}(\theta,\alpha(\theta))\alpha'(\theta),\end{equation} 
and
 \begin{eqnarray}\label{eq:partialF1}
 \frac{\partial F}{\partial \theta}(\theta,\alpha)&=&4\alpha(1-\alpha)\int y^2\frac{q^{\theta^*}(y)}{\left(\alpha e^{\theta y }+(1-\alpha)e^{-\theta y}\right)^2}\dd y \geq 0,\\ \label{eq:partialF2}
  \frac{\partial F}{\partial \alpha}(\theta,\alpha)&=&2\int y\frac{q^{\theta^*}(y)}{\left(\alpha e^{\theta y }+(1-\alpha)e^{-\theta y}\right)^2}\dd y.
 \end{eqnarray}
 Additionally, by taking derivatives with respect to $\theta$ in  \eqref{eq:optimalpha2} we have
 \begin{equation}\label{eq:alphaprime}
\alpha'(\theta)=-\frac{\partial G}{\partial \alpha}(\theta,\alpha(\theta))^{-1}\frac{\partial G}{\partial \theta}(\theta,\alpha(\theta)),
\end{equation}
and likewise, 
 \begin{eqnarray}\label{eq:partialG1}
 \frac{\partial G}{\partial \theta}(\theta,\alpha)&=&2\alpha(1-\alpha)\int y\frac{q^{\theta^*}(y)}{\left(\alpha e^{\theta y }+(1-\alpha)e^{-\theta y}\right)^2}\dd y ,\\ \label{eq:partialG2}
  \frac{\partial G}{\partial \alpha}(\theta,\alpha)&=&\int \frac{q^{\theta^*}(y)}{\left(\alpha e^{\theta y }+(1-\alpha)e^{-\theta y}\right)^2}\dd y> 0.
 \end{eqnarray}
The conclusion follows by replacing \eqref{eq:partialF1},\eqref{eq:partialF2},\eqref{eq:alphaprime},\eqref{eq:partialG1} and \eqref{eq:partialG2} in \eqref{eq:partialf} and invoking the Cauchy-Schwarz inequality.

We now show the second inequalities in~\eqref{monotone>theta} and~\eqref{monotone<theta}.
To this end, we will first show
\begin{equation} \label{eq:Falphacases} F(\theta,\alpha(\theta))\begin{cases}\geq F(\theta,0.5) & 0\leq \theta\leq \theta^*, \\ \leq F(\theta,0.5) & \theta\geq \theta^*.\end{cases}\end{equation}

Let $\phi$ denote the density of a standard Gaussian random variable.
We can write
\begin{align*}\frac{F(\theta,\alpha)-F(\theta,0.5)}{2\alpha-1} &= \int y \cdot \frac{\left(\alpha^*e^{\theta^* y}+(1-\alpha^*)e^{-\theta^* y}\right)}{\left(e^{\theta y}+e^{-\theta y}\right)\left(\alpha e^{\theta y}+(1-\alpha)e^{-\theta y}\right)}\phi(y)e^{-{\theta^*}^2/2}\dd y.\\
&=: \int y  \cdot \rho_{\theta, \alpha}(y) \dd y\,.
\end{align*}
It is straightforward to verify that for $\alpha, \alpha^* \geq 1/2$, if $\leq \alpha \leq \alpha^*$ and $\theta \leq \theta^*$, then
\begin{equation*}
\rho_{\theta, \alpha}(y) \geq \rho_{\theta, \alpha}(-y) \quad \forall y \geq 0.
\end{equation*}
On the other hand, if $\alpha \geq \alpha^*$ and $\theta \geq \theta^*$, then
\begin{equation*}
\rho_{\theta, \alpha}(y) \leq \rho_{\theta, \alpha}(-y) \quad \forall y \geq 0.
\end{equation*}
In particular, this yields that for $\alpha, \alpha^* \geq 1/2$,
\begin{equation*}
\frac{F(\theta,\alpha)-F(\theta,0.5)}{2\alpha-1} \begin{cases} \geq 0 & \text{if $\alpha \leq \alpha^*$ and $0 \leq \theta \leq \theta^*$} \\
\leq 0 & \text{if $\alpha \geq \alpha^*$ and $\theta \geq \theta^*$.}
\end{cases}
\end{equation*}
To complete the proof of~\eqref{eq:Falphacases}, we used the facts, proved in Lemma \ref{lemma:atheta} that $\alpha(\theta) \geq 1/2$ and that $\alpha(\theta) \leq \alpha^*$ if $0 \leq \theta \leq \theta^*$ and $\alpha(\theta) \geq \theta^*$ if $\theta \geq \theta^*$.

Finally, we show that the iterates $\theta_{EM_0}^t$ satisfy
\begin{equation}\label{eq:em0_iterates}
\theta_{EM_0}^{t+1} = F(\theta_{EM_0}^{t}, 0.5)\,.
\end{equation}
\Citet[Equation (3.1)]{DasTzaZam17} show
\begin{equation*}
\theta_{EM_0}^{t+1} = \E_{Y \sim \cN(\theta^*, 1)}[Y\tanh(\theta_{EM_0}^{t} Y)]\,.
\end{equation*}
Since $y \tanh(\theta_{EM_0}^{t} y)$ is an even function of $y$, this value is unchanged if we integrate with respect to the mixture $\alpha^* \cN(\theta^*, 1) + (1 - \alpha^*) \cN(-\theta^*, 1)$.
We obtain
\begin{equation*}
\theta_{EM_0}^{t+1} = \int y\tanh(\theta_{EM_0}^{t} y) q^{\theta^*}(y) \dd y\,,
\end{equation*}
and by comparing this to~\eqref{eq:def_f}, we immediately see that the right side is $F(\theta_{EM_0}^{t}, 0.5)$.

We can now show the second two inequalities in~\eqref{monotone>theta} and~\eqref{monotone<theta}.
We proceed by induction.
Let's first suppose $\theta^0 \geq \theta^*$.
Then indeed for $t = 0$, we have $\theta^* \leq \theta^t_{sEM} \leq \theta^t_{EM_0}$.
If this relation holds for some $t$, then we have
\begin{align*}
\theta^{t+1}_{sEM} & = F(\theta^{t}_{sEM}, \alpha(\theta^{t}_{sEM})) \\
& \leq F(\theta^{t}_{sEM}, 0.5) \\
& \leq F(\theta^{t}_{EM_0}, 0.5) \\
& = \theta^{t+1}_{EM_0}\,,
\end{align*}
where the first inequality uses~\eqref{eq:Falphacases}, the second uses the fact that $F$ is an increasing function in its first coordinate~\eqref{eq:partialF1}, and the final equality is~\eqref{eq:em0_iterates}.
The proof of the second inequality in~\eqref{monotone<theta} is completely analogous.
\end{proof}
\begin{proof}[Proof of Theorem \eqref{teo:mixture}, equation \eqref{eq:fast}]

  Suppose first $\theta^0> \theta^*$.
  In this case, it suffices to show that $F(\theta, \alpha) \leq F(\theta, \alpha^*)$ for all $\alpha \geq \alpha^*$ for all $\theta \geq \theta^*$.
  Indeed, we can then appeal to precisely the same argument as in the proof of Theorem~\eqref{teo:mixture}, equation~\eqref{emgeom_correct}, to compare the iterates of sEM (which satisfy $\theta_{sEM}^{t+1} = F(\theta_{sEM}^{t}, \alpha(\theta_{sEM}^{t}))$) to those of vEM (which satisfy $\theta_{sEM}^{t+1} = F(\theta_{sEM}^{t}, \alpha^*)$.)
  
  We have that
  \begin{align}
 \nonumber \frac{F(\theta,\alpha)-F(\theta,\alpha^*)}{2(\alpha-\alpha^*)}=&\int y\frac{q^{\theta^*}(y)}{\left(\alpha e^{\theta y} + (1-\alpha) e^{-\theta y}\right)\left(\alpha^* e^{\theta y} + (1-\alpha^*) e^{-\theta y}\right)}\dd y \\  \label{eq:fdif}=& \int_{y\geq 0} f_\theta(y)\dd y,
  \end{align}
where \begin{align*}f_\theta(y):=y\frac{g_\theta(y)\phi(y) e^{-{\theta^*}^2/2 }}{\left(\alpha e^{\theta y} + (1-\alpha) e^{-\theta y}\right)\left(\alpha^* e^{\theta y} + (1-\alpha^*) e^{-\theta y}\right)\left(\alpha e^{-\theta y} + (1-\alpha) e^{\theta y}\right)\left(\alpha^* e^{-\theta y} + (1-\alpha^*) e^{\theta y}\right)},\end{align*}
and \begin{align*}
g_\theta(y):=&\left(\alpha e^{-\theta y} + (1-\alpha) e^{\theta y}\right)\left(\alpha^* e^{-\theta y} + (1-\alpha^*) e^{\theta y}\right)q^{\theta^*}(y) \\ &-\left(\alpha e^{\theta y} + (1-\alpha) e^{-\theta y}\right)\left(\alpha^* e^{\theta y} + (1-\alpha^*) e^{-\theta y}\right)q^{\theta^*}(-y)\\ 
= & L \left(e^{y\left(2\theta-\theta^*\right)}-e^{-y\left(2\theta-\theta^*\right)}\right)+M\left(e^{y\theta^*}-e^{-y\theta^*}\right)+N \left(e^{y\left(2\theta+\theta^*\right)}-e^{-y\left(2\theta+\theta^*\right)}\right)
\end{align*}
with
\begin{eqnarray*}
L &=&(1-\alpha^*)^2(1-\alpha)- \alpha {\alpha^*}^2,\\
M &=& (2\alpha^*-1)(\alpha+\alpha^*-2\alpha\alpha^*), \\
N  &= & \alpha^*(1-\alpha^*)(1-2\alpha).
\end{eqnarray*}
Notice that for $y \geq 0$ and $\theta>\theta^*$ the three above differences of exponentials are positive, and that $e^{y\left(2\theta-\theta^*\right)}-e^{-y\left(2\theta-\theta^*\right)}\geq e^{y\theta^*}-e^{-y\theta^*}$. Moreover, if $1/2\leq\alpha^*,\alpha < 1$, then $N<0$ and $M>0$, and if furthermore $\alpha\geq\alpha^*$, then also $L< 0$.
Therefore,  
\begin{eqnarray}\nonumber g_\theta(y)&<& (L+M)\left(e^{y\theta^*}-e^{-y\theta^*}\right)\\ \label{eq:gleq}&=&(1-2\alpha)(3{\alpha^*}^2-3\alpha^*+1)\left(e^{y\theta^*}-e^{-y\theta^*}\right)\leq0. \end{eqnarray}
This proves that when $\alpha \geq \alpha^* \geq 1/2$, we have $F(\theta, \alpha) \geq F(\theta, \alpha^*)$, as claimed.

Now let's show that there exists a $\theta_{\text{fast}} < \theta^*$ such that if $\theta  \in [\theta_{\text{fast}}, \theta^*]$, then $F(\theta, \alpha) \geq F(\theta, \alpha^*)$ for all $\alpha^* \geq \alpha > 1/2$. (As above, this will suffice to prove the desired claim by applying the argument in the proof of Theorem~\eqref{teo:mixture}, equation~\eqref{emgeom_correct}.)
It suffice to show that for $\theta  \in [\theta_{\text{fast}}, \theta^*]$, we have
\begin{equation*}
g_\theta(y) \leq 0 \quad \forall y \geq 0\,.
\end{equation*}
First, we note that, since $N < 0$, for any $\theta > \theta^*/2$, the term $N \left(e^{y\left(2\theta+\theta^*\right)}-e^{-y\left(2\theta+\theta^*\right)}\right)$ is always eventually dominant, so there exists a $y^*$ such that
\begin{equation*}
g_\theta(y) <0 \quad \forall \theta > \theta^*/2, y > y^*\,.
\end{equation*}
It therefore suffices to focus on the compact interval $[0, y^*].$

To proceed, let us consider what happens when $\theta = \theta^*$.
Carrying out the exact same argument as above, we obtain that as long as $\alpha > 1/2$, we have
\begin{equation*}
g_{\theta^*}(y) \leq (L+M) (e^{y \theta^*} - e^{- y \theta^*}) \quad \forall y \geq 0,
\end{equation*}
where $L+M$ is negative.

Let us examine the derivative $\frac{\partial }{\partial \theta} g_\theta(y)$:
\begin{equation*}
\frac{\partial }{\partial \theta} g_\theta(y) = 2 y L(e^{y\left(2\theta-\theta^*\right)}+e^{-y\left(2\theta-\theta^*\right)}) + 2y N(e^{y\left(2\theta+\theta^*\right)}+e^{-y\left(2\theta+\theta^*\right)})\,.
\end{equation*}

We conclude that if $\theta' < \theta^*$ is such that
\begin{equation*}
(\theta^* - \theta') (4 |L| y e^{y \theta^*}+ 4 |N|y e^{3y \theta^*})  \leq - (L+M) (e^{y \theta^*} - e^{- y \theta^*})\,,
\end{equation*}
then
\begin{equation*}
\left|\frac{\partial }{\partial \theta} g_\theta(y)\right| \cdot (\theta^* - \theta') \leq -g_{\theta^*}(y)\,,
\end{equation*}
which yields
\begin{equation*}
g_{\theta'}(y) = g_{\theta^*}(y) - \int_{\theta'}^{\theta^*} \frac{\partial }{\partial \theta} g_\theta(y) \dd \theta \leq g_{\theta^*}(y) +  \left|\frac{\partial }{\partial \theta} g_\theta(y)\right| \cdot (\theta^* - \theta') \leq 0\,.
\end{equation*}
Hence, if we define
\begin{equation*}
\delta = \inf_{y \in [0, y^*]} \frac{-(L+M)(e^{y \theta^*} - e^{- y \theta^*})}{4 |L| y e^{y \theta^*}+ 4 |N|y e^{3y \theta^*}}\,,
\end{equation*}
then as long as this quantity is positive, we can take $\theta_{\text{fast}} = \theta^* - \delta$.
But positivity follows immediately from the fact that this function is continuous and positive on $(0, y^*]$ and has a positive limit as $y \to \infty$.
\end{proof}

\section{Technical results}

\subsection{Intermediate results for Theorem \ref{teo:mixtureloss} and \ref{teo:mixture}}
\begin{proposition}\label{prop:deri}
For the asymmetric mixture of two Gaussians \eqref{eq:mixtureg} we have that for all $\theta<0$
\begin{equation} \label{eq:ineqLl}L'_{\alpha^*}(\theta)>\ell'_{\alpha^*}(\theta)\end{equation} 
\end{proposition}
\begin{proof}
Let us write $\ell(\theta, \alpha) = - \E_{Y \sim Q^{\theta^*}} \log q^{\theta, \alpha}(Y)$ for the expected log-likelihood function in the overparametrized model
\begin{equation*}
q^{\theta, \alpha} = \alpha \mathcal N(y; \theta, 1) + (1- \alpha) \mathcal N(y; -\theta, 1)\,.
\end{equation*}
We then have
\begin{equation}\label{eq:partiall} \frac{\partial}{\partial \theta} \ell(\theta, \alpha) = \int y\left[\frac{\alpha e^{-(\theta-y)^2/2}-(1-\alpha) e^{-(\theta+y)^2/2}}{\alpha e^{-(\theta-y)^2/2}+(1-\alpha) e^{-(\theta+y)^2/2}}\right]q^{\theta^*,\alpha^*,}(y)\dd y-\theta = F(\theta, \alpha) - \theta\,,\end{equation}
where $F$ is defined in~\eqref{eq:def_f}.

We have $\ell'_{\alpha^*}(\theta) = \frac{\partial}{\partial \theta} \ell(\theta, \alpha^*) = F(\theta, \alpha^*) - \theta$.
Likewise, if we recall that $L_{\alpha^*}(\theta)=L_2(\theta,w_\theta)$ where  $w_\theta$ satisfies \eqref{eq:semidual1}, then we have
\begin{equation}\label{eq:partialLmix} L'_{\alpha^*}(\theta) =\frac{\partial L_2}{\partial \theta}(\theta,w_\theta)=\frac{\partial }{\partial \theta}\ell(\theta,\alpha(\theta)) = F(\theta, \alpha^*) - \theta\,.,
\end{equation}
where the penultimate equality is obtained by differentiating~\eqref{eq:l2}and using the definition of $\alpha(\theta)$ in \eqref{eq:optimalpha2}.
Then, to establish \eqref{eq:ineqLl} it suffices to show that for each $\theta<0$,
\begin{equation*}
F(\theta, \alpha^*) < F(\theta, \alpha(\theta))\,.
\end{equation*}
Moreover, since Lemma \ref{lemma:atheta} shows that $\alpha(\theta)> \alpha^*>0.5$, it suffices to show that $F(\theta, \alpha)$ is a strictly increasing function of $\alpha$ for $\alpha \geq 0.5$.

Recall~\eqref{eq:partialF2}, which shows
\begin{equation*}
\frac{\partial F}{\partial \alpha}(\theta,\alpha)=2\int y\frac{q^{\theta^*}(y)}{\left(\alpha e^{\theta y }+(1-\alpha)e^{-\theta y}\right)^2}\dd y\,.
\end{equation*}

Since $\alpha^* > 0.5$, we have $q^{\theta^*}(y) > q^{\theta^*}(-y)$ for all $y > 0$.
Furthermore, for $\theta \leq 0$ and $\alpha \geq 0.5$, it holds
\begin{equation*}
\frac{1}{\left(\alpha e^{\theta y }+(1-\alpha)e^{-\theta y}\right)^2} \geq \frac{1}{\left(\alpha e^{-\theta y }+(1-\alpha)e^{\theta y}\right)^2} \quad \forall y > 0\,.
\end{equation*}
Therefore
\begin{align*}
\frac{\partial F}{\partial \alpha}(\theta,\alpha)& =2\int y\frac{q^{\theta^*}(y)}{\left(\alpha e^{\theta y }+(1-\alpha)e^{-\theta y}\right)^2}\dd y \\
& = 2\int_{y > 0} y\frac{q^{\theta^*}(y)}{\left(\alpha e^{\theta y }+(1-\alpha)e^{-\theta y}\right)^2}\dd y + 2\int_{y > 0} (-y)\frac{q^{\theta^*}(-y)}{\left(\alpha e^{-\theta y }+(1-\alpha)e^{\theta y}\right)^2}\dd y > 0\,,
\end{align*}
which proves the claim.
\end{proof}

\begin{lemma}\label{lemma:atheta}
Suppose $\alpha^*>1/2$
\begin{itemize}
\item[(a)]for any $\theta$, $\alpha(\theta)> 1/2.$

\item[(b)] $\alpha(\cdot)$ is decreasing (increasing) whenever $\theta<0$ ($\theta>\theta^*$) and in either case $\alpha(\theta)\geq \alpha^*$. Moreover, $\lim_{\|\theta\|\rightarrow \infty}\alpha(\theta)=1$.
\item[(c)]$\alpha(\theta)\leq \alpha^*$ whenever $0\leq \theta\leq \theta^*$.

\end{itemize}\end{lemma}

\begin{proof}
We begin by recalling the function $G$, defined in~\eqref{eq:optimalpha3}.
By~\eqref{eq:partialG2}, this function is a strictly increasing function of $\alpha$; therefore, $\alpha(\theta)$ is the unique number in $[0, 1]$ satisfying
\begin{equation*}
G(\theta, \alpha(\theta)) = \alpha^*
\end{equation*}
and $\alpha(\theta) > p$ if and only if $G(\theta, p) < \alpha^*$.
Let us first prove (a).
It suffices to show that $G(\theta, 1/2) < \alpha^*$.
Write $\phi_{\theta^*}$ for the density of $\cN(\theta^*, 1)$.
We then have
\begin{equation*}
G(\theta, 1/2) = \int \frac{e^{\theta y}}{e^{\theta y} + e^{-\theta y}} q^{\theta^*}(y) \dd y= \int \frac{\alpha^* e^{\theta y} + (1-\alpha^*) e^{-\theta y}}{e^{\theta y} + e^{-\theta y}} \phi_{\theta^*}(y) \dd y\,.
\end{equation*}
But if $\alpha^* > 1/2$, then $\frac{\alpha^* e^{\theta y} + (1-\alpha^*) e^{-\theta y}}{e^{\theta y} + e^{-\theta y}} < \alpha^*$ for all $\theta$ and $y$.
Since $\phi_{\theta^*}(y)$ is a probability density, we obtain that $G(\theta, 1/2) < \alpha^*$, as desired.

It is straightforward to see that $\alpha(0) = \alpha(\theta^*) = \alpha^*$.
To show monotonicity,  we rely on the formula \eqref{eq:alphaprime} for $\alpha'(\theta)$. 
If $\theta<0$ the conclusion is a direct consequence of \eqref{eq:alphaprime} and Lemma \ref{lemma:h}(b). 
If $\theta > \theta^*$, the conclusion follows similarly from Lemma \ref{lemma:h}(c) but the argument is more delicate, as applying this lemma requires that $\alpha(\theta)\geq \alpha^*$. 
Suppose that there exists a $\theta > \theta^*$ for which $\alpha'(\theta) < 0$.
Let us denote by $\theta_0$ the infimum over all such $\theta$.
By \eqref{eq:alphaprime}, $\frac{\partial{G}}{\partial \theta}(\theta_0,\alpha(\theta_0))$ must be therefore nonnegative, which by Lemma~\ref{lemma:h}(c) implies that $\alpha(\theta_0) < \alpha^* = \alpha(\theta^*)$. But since $\alpha'(\theta) \geq 0$ for all $\theta \in [\theta^*, \theta_0)$, this is a contradiction.
Therefore $\alpha'(\theta) \geq 0$ for all $\theta \geq \theta^*$, as claimed.

Finally, the limit statement follows from the dominated convergence theorem. Since $\alpha^* = G(\theta, \alpha(\theta))$ for all $\theta \in \R$, it holds
\begin{align*}
\alpha^* & = \lim_{\|\theta\| \to \infty} G(\theta, \alpha(\theta)) \\
& = \int \lim_{\|\theta\| \to \infty} \frac{\alpha(\theta) e^{\theta y}}{\alpha(\theta) e^{\theta y} + (1 - \alpha(\theta))e^{-\theta y}} q^{\theta^*}(y) \dd y\,,
\end{align*}
where the second inequality is by the dominated convergence theorem.
Since $\alpha(\cdot)$ is monotonic outside the interval $[0, \theta^*]$, as $\alpha(\theta)$ has a limit as $\theta \to + \infty$ or $\theta \to - \infty$.
Let us first consider $\theta \to \infty$ (the negative case is exactly analogous).
If this limit is different from $1$, then
\begin{equation*}
\lim_{\theta \to \infty} \frac{\alpha(\theta) e^{\theta y}}{\alpha(\theta) e^{\theta y}+ (1 - \alpha(\theta))e^{-\theta y}} = \begin{cases} 1 & y > 0 \\ 0 & y < 0 \end{cases}\,.
\end{equation*}
But this is a contradiction, since $\alpha^* \neq \int_{y \geq 0} q^{\theta^*}(y) \dd y$ if $\alpha^* > 1/2$.
This proves the claim.

Let's now prove (c). Notice it suffices to show that (i) $\alpha'(0)<0$ and (ii) the only solutions to the equation $\alpha(\theta)=\alpha^*$ are $\theta=0$ and $\theta=\theta^*$. 
The first claim is a simple consequence of \eqref{eq:alphaprime} and Lemma~\ref{lemma:h}(b).

The second claim is a bit more involved. Suppose $\alpha(\theta)=\alpha^*$. By simple algebra (as in the proof of theorem \ref{teo:mixture}) it can be show then the following relation holds
$$\int_{y\geq 0} \frac{2\alpha^*(1-\alpha^*)(2\alpha^*-1)e^{-{\theta^*}^2}\left(e^{2\theta y}-1\right)\left(e^{2\theta^*y} -e^{2\theta y}\right)}{e^{(\theta^*+2\theta)y}\left(\alpha^* e^{\theta y}+(1-\alpha^*)e^{-\theta y}\right)\left((1-\alpha^*) e^{\theta y}+\alpha^* e^{-\theta y}\right)}\phi(y)\dd y=0.$$ 
The integral above can only be zero if $\theta=0$ or $\theta=\theta^*$, otherwise the integrand is either positive or negative for each value of $y\geq0$. This concludes the proof.
\end{proof}

\begin{lemma}\label{lemma:h} Suppose $\alpha^*> 0.5$.
Let
$$G_\theta(\theta,\alpha):=\frac{1}{2\alpha(1-\alpha)}\frac{\partial{G}}{\partial \theta}(\theta,\alpha)=\int \frac{y}{\left(\alpha e^{\theta y}+(1-\alpha)e^{-\theta y}\right)^2}q^{\theta^*}(y)\dd y.$$
Then,
\begin{itemize}
\item[(a)] For each $\theta\geq 0$, $G_\theta$ is a decreasing as function of $\alpha$. Conversely, for each $\theta\leq 0$, $G_\theta$ is an increasing function of $\alpha$.
\item[(b)]$G_\theta(\theta,\alpha)\geq 0$ if $\theta\leq 0$ and $\alpha > 1/2$.
\item[(c)]$G_\theta(\theta,\alpha)\leq0$ if $\theta\geq \theta^*$ and $\alpha\geq \alpha^*$.
\end{itemize}
\end{lemma}
\begin{proof}
To see (a), notice that 
$$\frac{\partial G_\theta}{\partial \alpha}(\theta,\alpha)=-2\int  \frac{y\left( e^{\theta y}-e^{-\theta y}\right)}{\left(\alpha e^{y\theta}+(1-\alpha)e^{-y\theta}\right)^3}q_{\alpha^*,\theta^*}(y)\dd y.$$
The integrand is either positive (if $\theta>0$) or negative (if $\theta<0$) for each $y$, and the conclusion follows.

To prove (b) and (c), we note that~\eqref{eq:partialF2} implies that
\begin{equation*}
G_\theta(\theta, \alpha) = \frac 12 \frac{\partial F}{\partial \alpha}(\theta, \alpha)\,.
\end{equation*}
But we have already shown in the proof of Proposition~\ref{prop:deri} that $\frac{\partial F}{\partial \alpha}(\theta, \alpha) > 0$ for all $\theta \leq 0$ and $\alpha > 1/2$. This proves (b).

Likewise, the proof of Theorem \eqref{teo:mixture}, equation \eqref{eq:fast}, shows that $F(\theta, \alpha) \leq F(\theta, \alpha^*)$ for all $\theta\geq \theta^*$ and $\alpha\geq \alpha^*$. This proves that $G_\theta(\theta, \alpha^*) = \frac 12 \frac{\partial F}{\partial \alpha}(\theta, \alpha^*) \leq 0$.
To conclude, we appeal to part (a): since $\theta\geq\theta^*>0$, $G_\theta$ is decreasing as a function of $\alpha$, and hence, $G_\theta(\theta,\alpha)\leq G_\theta(\theta,\alpha^*)\leq0$.
\end{proof}

\begin{figure}[H]
\includegraphics[width=1.0\textwidth]{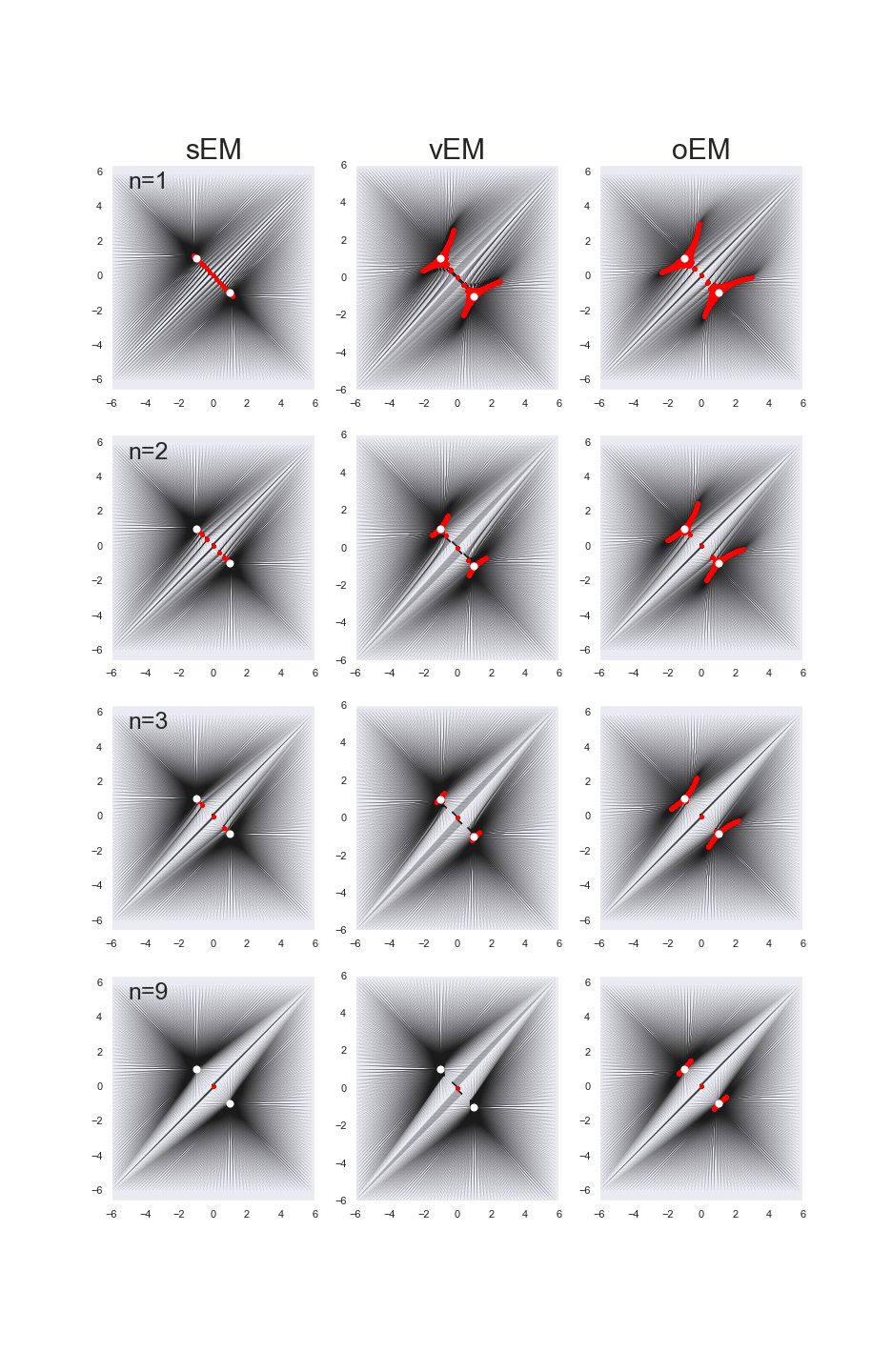}.
\caption{Evolution of iterates for the equal mixture of two Gaussians example, for every starting value in the grid. The flow from the initial values to current ones is indicated by the greyscale flow, and with red dots. While dots represent true global optima. Each row corresponds to a different iteration number $n=1,2,3,9$.}
\label{fig:example2equal}
\end{figure}
\section{Experiment details and supplemental experiment discussion}

In all experiments of section \ref{sec:simulation} we performed $2000$ iterations of each algorithm from each initial value. Visual exploration of iterates revealed convergence in all simulations after far fewer than $2000$ iterations.

For Sinkhorn EM, at each call of Sinkhorn algorithm we performed a number of $200$ row and column normalization steps \citep{Peyre2019}, starting from zero potentials. Notice it would be possible to use warm starts in the sEM outer loop calculation, re-using the potentials from previous iteration. We leave this for future work.

In all cases we consider as error metric the estimation error, defined as the squared 2-Wasserstein distance $W^2_2(P^*,P^{f})$ between the true mixture $P^*=\frac{1}{}\sum_{k=1}^K\alpha^*_k \delta(\theta^*_k)$ ($K:=|\cX|$)and the one with final values $P^f=\frac{1}{k}\sum_{k=1}^K \alpha^f_k\delta(\theta^f_k)$. We compute such distances with the \texttt{emd2} function \texttt{Python POT} package \citep{Flamary2017pot}. For the overparameterized EM algorithm $\alpha^f_k\neq \alpha^*_k$ which creates numerical stabilities that we avoid by approximating the 2-Wasserstein distance with the Sinkhorn algorithm, using the \texttt{sinkhorn} function on the same package, with regularization parameter \texttt{reg}$=0.1$ 

\subsection{Symmetric mixture of two Gaussians with asymmetric weights}

Here, with $\theta^*=1$ and $\alpha^*$ takes values on  $51$-length grid $[0.5,1.0]$. A number of $n=1000$ data points were sampled. We studied the evolution of each algorithm for starting values $\theta^0$ on a $26$-size grid $[-2,-2]$ (26 elements). All the presented results are averages over a number of 10 sampled datasets. In the two right plots of Fig. \ref{fig:example1} we defined the number of iterations required to converged as the least iteration number such that the approximation error at that iteration is smaller than $1.5$ times the error at the final iteration. For the overparameterized EM algorithm, we always used an initial weight $\alpha^0=0.5$, as in \cite{Xu2018}. An IPython notebook that reproduces the findings of Fig. \ref{fig:example1} is \href{http://github.com/gomena/SinkhornEM_Public}{available online}.

\begin{figure}[H]
\includegraphics[width=1.0\textwidth]{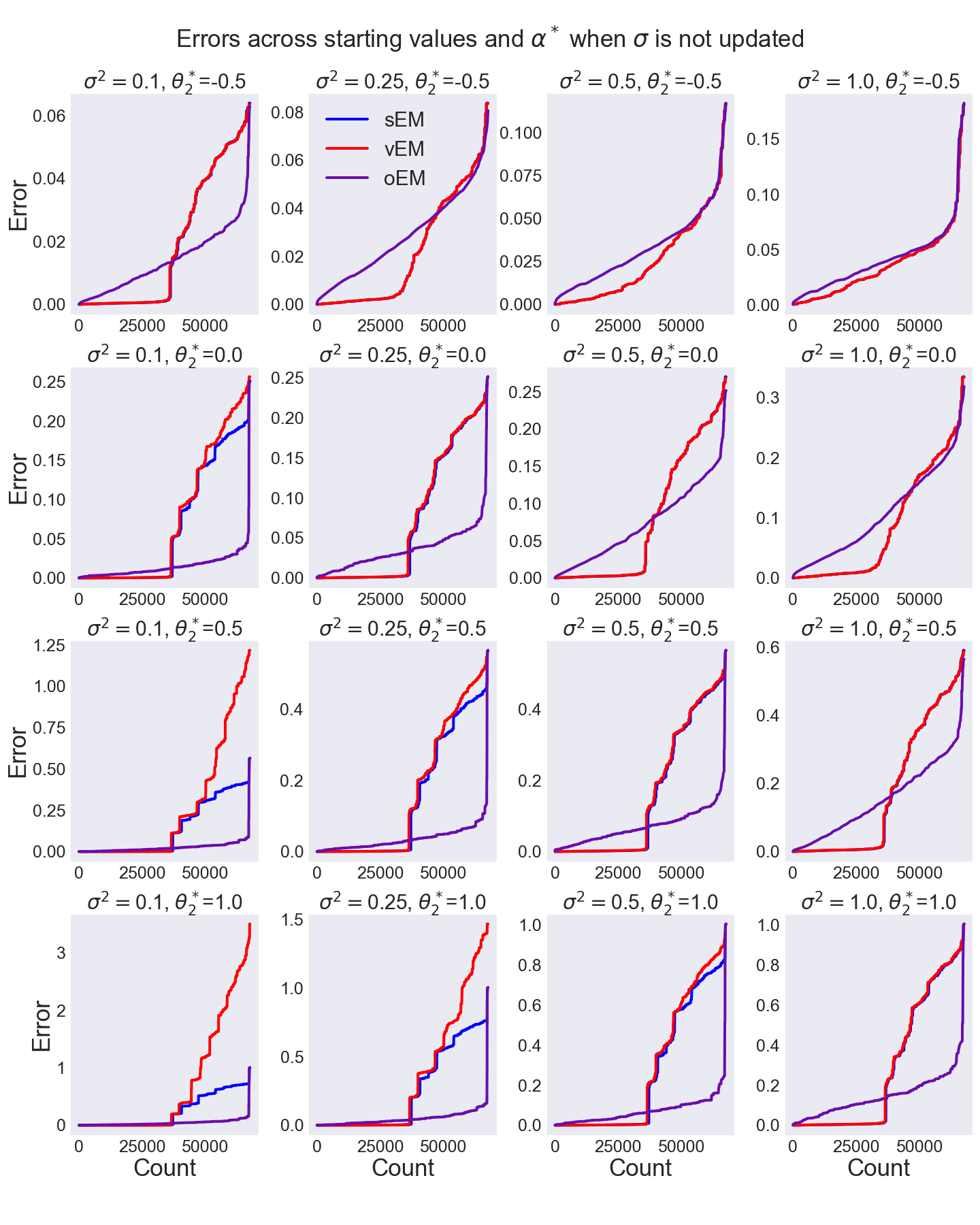}.
\caption{Sorted errors as a function of parameters when $\sigma$ is not updated}
\label{fig:example2anpsigma}
\end{figure}

\begin{figure}[H]
\includegraphics[width=1.0\textwidth]{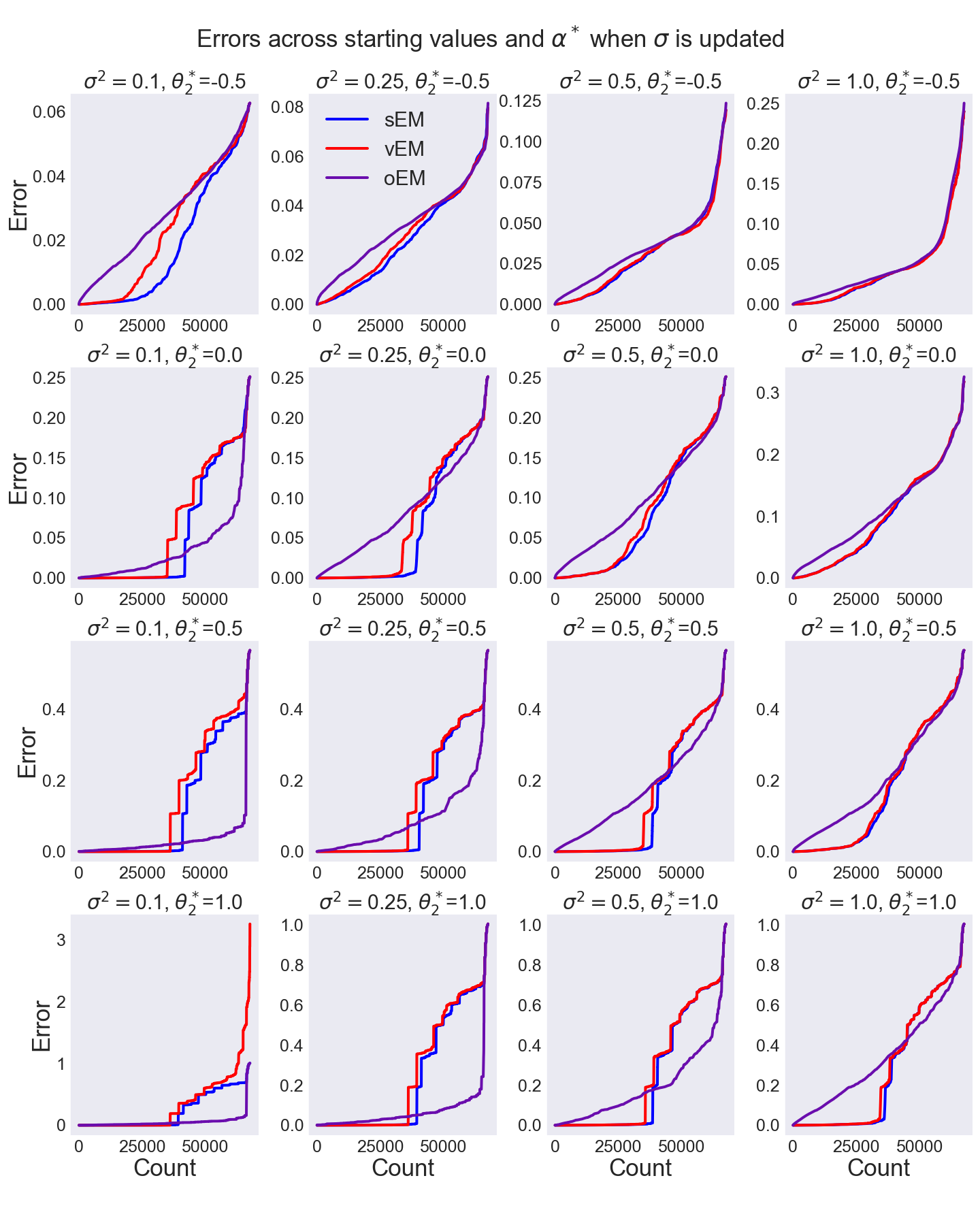}.
\caption{Sorted errors as a function of parameters when $\sigma$ is updated}
\label{fig:example2bsigma}
\end{figure}

\begin{figure}[H]
\includegraphics[width=1.0\textwidth]{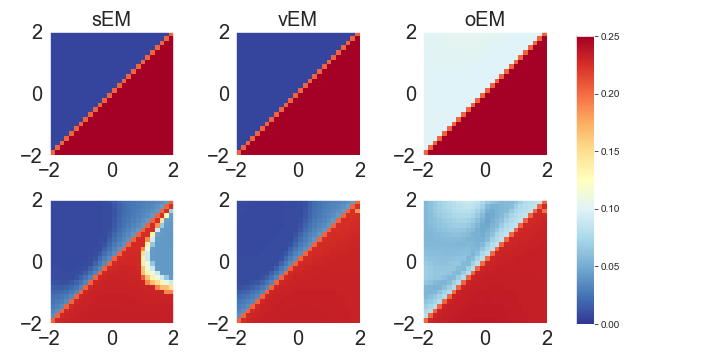}.
\caption{Final errors when sigma is not updated (first row) and updated (second row) , $\alpha^*=0.8$, $(\theta^*_1,\theta^*_2)=(-1,0.5)$ and $\sigma^*=1$. Overparameterized EM has significantly more errors in either case. Moreover, updating $\sigma$ leads to significant improvements of Sinkhorn EM over vanilla EM: there is an entire region of initial values (large $\theta^0_1$) for which convergence to the true value is achieved, unlike vanilla EM and overparameterized EM.}
\label{fig:example2csigma}
\end{figure}

\subsection{Equal mixture of two Gaussians}
We  used a $26\times 26$-size grid of  $[-2,2]\times[-2,2]$ for initial values $(\theta^0_1, \theta^0_2)$. Fig.  \ref{fig:example2equal} supplements the findings shown in Fig. \ref{fig:example2} in the main text.

\subsection{General mixture of two Gaussians}

The previous experiment was replicated with the following comprehensive choices of parameters.
\begin{itemize}
    \item ${\sigma^*}^2 \in \{0.1,0.25,0.5,1.0\}$.
    \item $\theta^*=(\theta^*_1,\theta^*_2)$ with $\theta^*_1=-1$ and $\theta^*_2 \in \{-0.5,0,0.5, 1.0\}$.
    \item $\alpha^*\in\{0.5, 0.55,0.6,0.65,0.7,0.75,0.8,0.85,0.9,0.95\}$.
\end{itemize}
For each of the $4\times 4\times 10$ above parameter configurations we sampled a number of $10$ datasets, each with a sample size of $n=1000$.  For the overparameterized EM algorithm, initial weights were also chosen as $\alpha^0=0.5$. Additionally, we analysed the cases were i) $\sigma^*$ is fixed or ii) treated as a parameter; i.e., updated at each iteration. For the later case, we always used the true value of $\sigma^*$ as the initial value. 

Thus, results in Fig.~\ref{fig:example2} correspond to the case $\sigma^*=1, \theta^*=(\theta^*_1,\theta^*_2)=(-1,1),\alpha^*=0.5$ and where $\sigma^*$ is not updated. More comprehensive results are presented in Figs. \ref{fig:example2anpsigma} ($\sigma^*$ fixed) \ref{fig:example2bsigma} ($\sigma^*$ updated), showing (sorted) errors across all simulations, starting points and true $\alpha^*$ for different values of $\sigma^*$ and $\theta^*$. The main conclusion is that Sinkhorn EM typically leads to smaller error tan vanilla EM, but there is a mixed behavior with overparameterized EM: there, errors may distribute more uniformly across possibilities, and results may be better or worse than Sinkhorn EM and vanilla EM, depending on the situation. 

Interestingly, when $\sigma^*$ is updated the performance of Sinkhorn EM may improve over vanilla EM, while inferences with overaparameterized-EM worsens. This can be seen by comparing Figs.  \ref{fig:example2anpsigma} and \ref{fig:example2bsigma}, and is further depicted in Fig.~\ref{fig:example2csigma}. There, we show final errors for each initial value at a particular configuration of $\sigma^*,\alpha^*,$ and $\theta^*$.

\begin{figure}[H]
\includegraphics[width=1.0\textwidth]{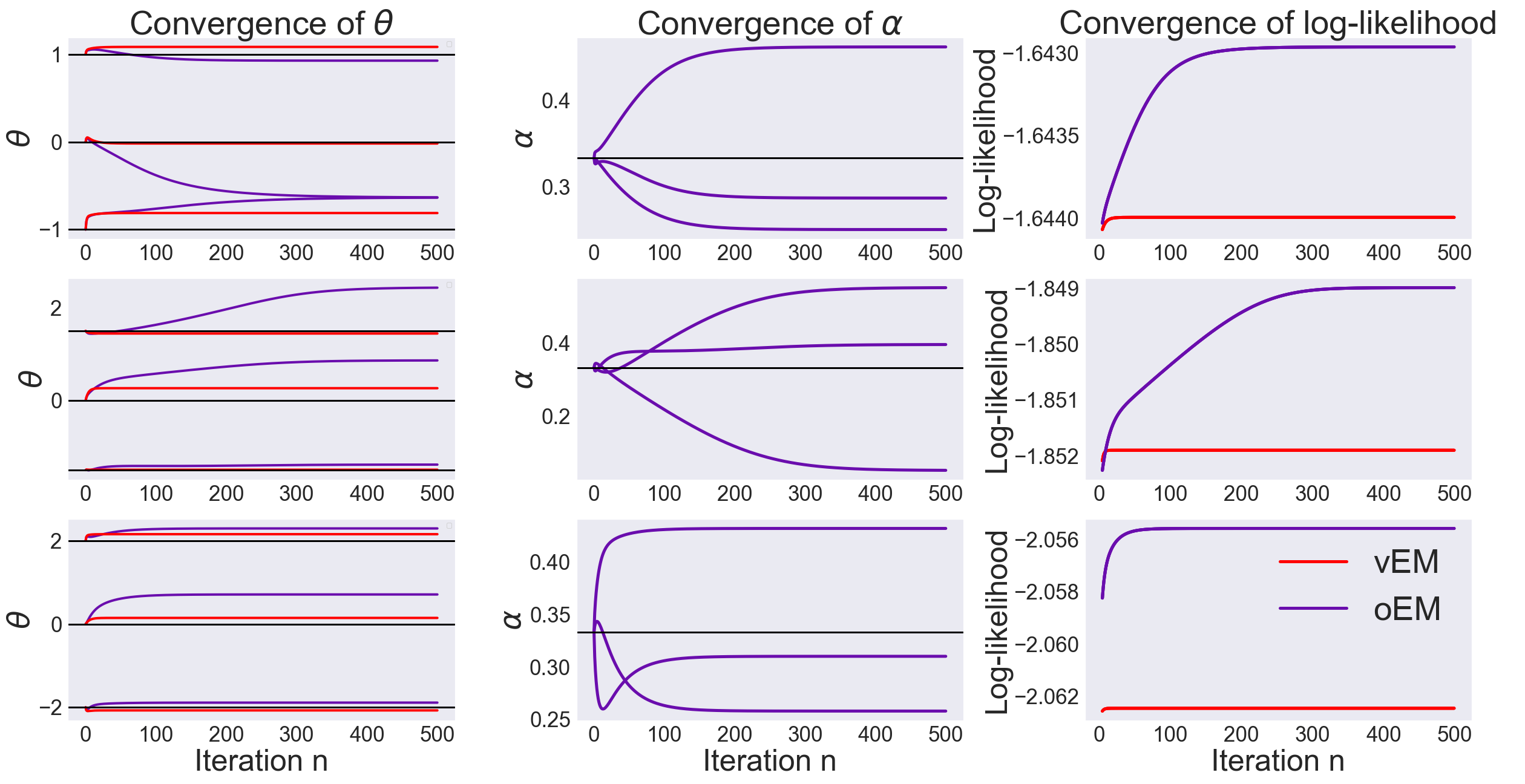}
\caption{Examples of sequences of iterates from different starting $\theta^*=(-\mu,0,\mu)$ (rows) show overparameterized EM may converge to undesirable solutions with nonetheless a slightly better likelihood. On each case, initial iterates are chosen as true values. \textbf{Left}: Evolution of each of the $\theta_i$ parameters, for vanilla EM  and overparameterized EM. Sinkhorn EM is not shown as its behaviour is distinguishable from vanilla EM for large iteration number. Black lines indicate true parameters. \textbf{Center}. Evolution of $\alpha_i$ parameters. \textbf{Right}. Evolution of the log-likelihood.}
\label{fig:subex31}
\end{figure}

 \subsection{Mixture of three Gaussians}
 Results in Fig.~\ref{fig:example3a}-\textbf{C} summarize many experiments, each with $n=500$ samples. Specifically, we considered 20 sampled datasets, and for each of them, initial $\theta^0$ were chosen as the true $\theta^*=(-\mu,0,\mu)$ plus a randomly-sampled Gaussian corruption at each component, with variances $\sigma^2_{noise}\in \{0,0.25,0.5,0.75,1.0\}$. Additionally, for overparameterized EM we considered an initial $\alpha^*$ equal to the true uniform $(1/3,1/3,1/3)$ or randomly sampled (uniformly) from the simplex. In the later case, we considered four samples. Therefore, Fig.~\ref{fig:example3a}-\textbf{C} summarizes 100 experiments for vanilla EM, Sinkhorn EM and overparameterized EM (true), and 400 experiments for overparameterized EM (random).

Fig.~\ref{fig:subex31} illustrates why overaparameterized EM has more error: when separation is small, iterates may often land into stationary points that have don't correspond to the true model, even with a better log-likelihood. For example, the first row of Fig.~\ref{fig:subex31} shows a case of mode collapse.

We attribute this type of failure to the fact the sample size is always finite. Notice the global convergence results of \cite{Xu2018} are stated in the population case. Our results suggest the population analysis may conceal important differences that only reveal themselves in challenging (e.g. small separation), finite sample setups.

Fig.~\ref{fig:subexp32} shows an additional experiment supporting our finite-sample hypothesis. We performed the same analysis as the one shown with Fig.~\ref{fig:example3a}, but for different sample sizes. They show that all errors decrease as $n$ increases, but errors with overaparameterized EM persist even with a $n=2000$ sample size.

\begin{figure}[H]
\hspace{-2.5cm}
\includegraphics[width=1.2\textwidth]{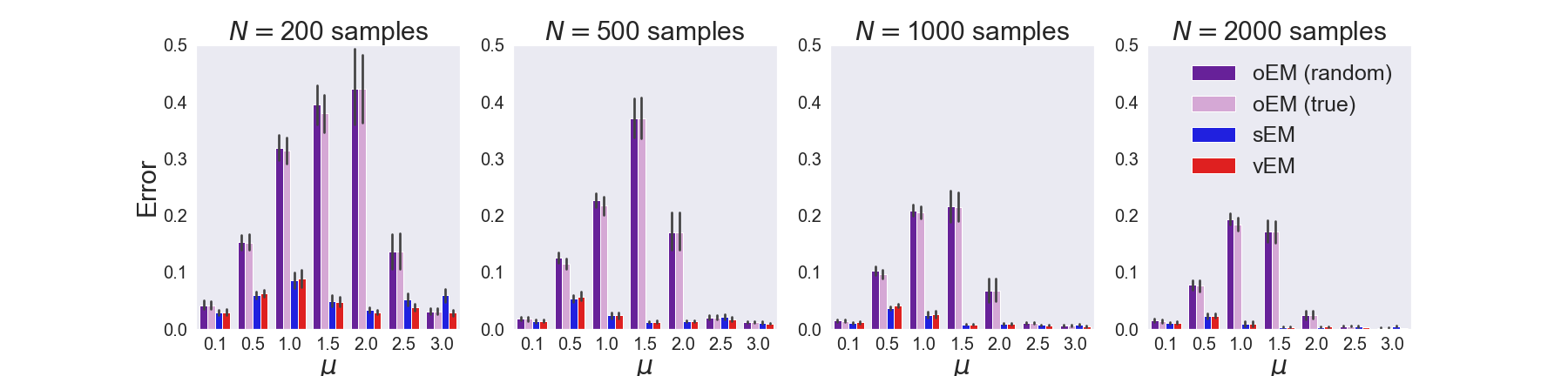}
\caption{Errors for the three mixture example, for different sample sizes.}
\label{fig:subexp32}
\end{figure}

Finally, in Fig.~\ref{fig:subexp33} we show a direct comparison with the results of \cite{Xu2018} (specifically, case 1 of example 3.3), and reconcile their findings with ours. On a three-mixture example \cite{Xu2018} showed overaparameterized-EM outperforms vanilla EM.  We show that while we are able to recover this behavior, the pattern is completely reversed by slightly modifying the example, now with both vanilla EM and Sinkhorn EM outperforming overparameterized EM. 

This example corresponds to a slight modification of our three-mixture example. Specifically, mixture components are now two dimensional (independent standard) Gaussians with $\theta^*_1=(-3,0), \theta^*_2=(0,0), \theta^*_3=(2,0)$ and mixture weights $\alpha^*=(0.5,0.3,0.2)$. We expand this example to study the effect of separation, by weighting each $\theta^*_i$ by a scaling factor of $\rho\in \{1,0.75,0.5,0.25\}$. Fig.~\ref{fig:subexp33} shows that, consistent with \cite{Xu2018}, overparameterized EM has the best performance when $\rho=1$. However, this pattern is completely reversed if separation is decreased, so eventually both Sinkhorn EM and vanilla EM outperform overparameterized EM.
\begin{figure}[H]

\includegraphics[width=1.0\textwidth]{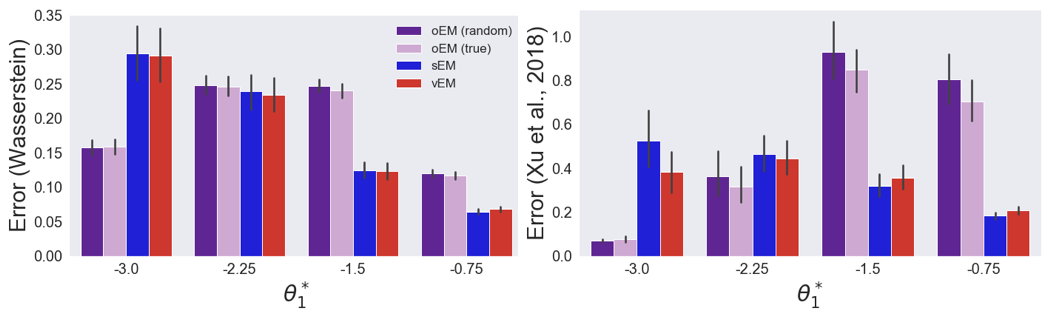}
\caption{We reproduced case 1 of experiment 3.3. in \cite{Xu2018}, but added different separations between components. The original case corresponds to the first columns ($\theta_1^*=(-3,0)$). The difference between left and right plots is the way errors are computed: the error definition in \cite[equation 11]{Xu2018} only takes into account the true weights $\alpha^i$ (but not the inferred $\alpha^f$), hence, it cannot capture label switching errors. Although results with these two error definitions are indeed different, our results show the overall dominance pattern is robust to the way errors are defined.}
\label{fig:subexp33}
\end{figure}

\section{Application to inference of neurons in \textit{C. elegans}: experimental details}
We employ the EM algorithm and its three variants, sEM, oEM, and vEM to undertake an image segmentation task in fluorescence microscopy images of the model organism \textit{C. elegans}. Images were captured via a spinning-disk confocal microscope with resolution (x,y,z)=(0.27,0.27,1.5) microns. Whole-brain calcium activity was measured using the fluorescent sensor GCaMP6s in animals expressing a stereotyped fluorescent color map that permitted class-type identification of every neuron in the worm's brain (NeuroPAL) ~\cite{yemini2019neuropal}.

Given the pixel locations and colors, and an atlas~\cite{yemini2019neuropal} of that encodes a prior on the cluster centers, we aim to infer the memberships of pixels into clusters that capture the shape and boundary of each neuron.

To assess the convergence rates and segmentation quality of the compared methods, we utilize several different optimization configurations:
\begin{enumerate}
    \item Randomly initialized cluster centers (we initialize the $\mu_k$'s by randomly choosing $K$ of the data points, this initialization scheme is similar to the strategy followed in~\cite{Xu2018}.) vs. cluster centers initialized at the atlas priors
    \item Fixed covariance matrices vs. updating the covariance in the EM routines
    \item Initializing the covariance matrix on the ground truth values vs. random initialization
\end{enumerate}

After we select a particular configuration, for example, "random center initialization $\times$ fixed covariance matrix $\times$ ground truth covariance initialization, we optimize using the three compared EM routines using a fixed time budget. This is to enable a fair wall-clock based convergence comparison between the methods. For all of the experiments, the time budget is set at 1 second.

We evaluate the segmentation performance using the following metrics:
\begin{itemize}
    \item \textbf{Training and testing log-likelihood} defined as $\sum_{n \in \mathcal{X}_{\text{epoch}}} \log P(X_n,\boldsymbol{\mu}_{1:K},\boldsymbol{\Sigma}_{1:K})$ for $\text{epoch} \in \{\text{train}, \text{test}\}$ respectively (the pixels are divided into a training set with \%80 of pixels used to fit the parameters and evaluate the convergence and a test set with \%20 of the pixels used to evaluate the goodness of fit properties.
    \item \textbf{Accuracy} is defined as the fraction of cell centers that are within the range of $3 \mu m$ from their true value. 
    \item \textbf{Mean square error (MSE)} is defined as $\frac{1}{K}\sum_{k=1}^K \norm{\boldsymbol{\mu}_k - \boldsymbol{\hat{\mu}}_k}^2$ where $\boldsymbol{\hat{\mu}}_k$ is the true value of cell center $k$.
\end{itemize}
We sample 5000 pixels from the GMM detailed in the following generative process. Starting from $K$ atlas neurons (here \textit{PDA, DVB, PHAL, ALNL, PLML}).
\begin{gather*}
    \boldsymbol{\mu}_k|\boldsymbol{\mu}_k^a,\boldsymbol{\Sigma}_k^a \sim \mathcal{N}(\boldsymbol{\mu}|\boldsymbol{\mu}_k^a,\boldsymbol{\Sigma}_k^a) \\
    \boldsymbol{\Sigma}_k = \sigma_k \mathbb{I}_6 \quad \sigma_k \sim \text{LogNormal}(1,.1) \\
    Z \sim \text{Categorical}(\frac{1}{K},\dots,\frac{1}{K}) \\
    \boldsymbol{Y_i}|Z,\boldsymbol{\mu}_k,\boldsymbol{\Sigma}_k,\boldsymbol{\mu}_k^a,\boldsymbol{\Sigma}_k^a \sim \mathcal{N}(\boldsymbol{Y}|\boldsymbol{\mu}_Z,\boldsymbol{\Sigma}_Z)
\end{gather*}
Notice that the sample space is the 6-dimensional spatio-chromatic space and the generated samples contain a location and an RGB color for each pixel. In figures ~\ref{fig:worm_TTF_config},\ref{fig:worm_TFT_config},\ref{fig:worm_TFF_config},\ref{fig:worm_FFF_config},\ref{fig:worm_FFT_config},\ref{fig:worm_FTF_config}, we display the evaluation metrics for the optimization configurations we have considered. In addition to the evaluation metrics, we also provide a visualization of the segmentation quality for the three methods by taking the average segmentation maps over multiple re-runs of the algorithms. In all cases, sEM obtains a slightly higher training and testing log-likelihoods. However, the segmentation accuracy and mean-squared-error is drastically improved over oEM and vEM, which tend to get stuck in local minima and yield poor segmentations on average.

The convergence behavior of the three algorithms through iterations on an individual example run is \href{http://github.com/gomena/SinkhornEM_Public}{available online} as animated GIF files. 

\begin{figure}[H]
    \centering
    \includegraphics[width=1\linewidth]{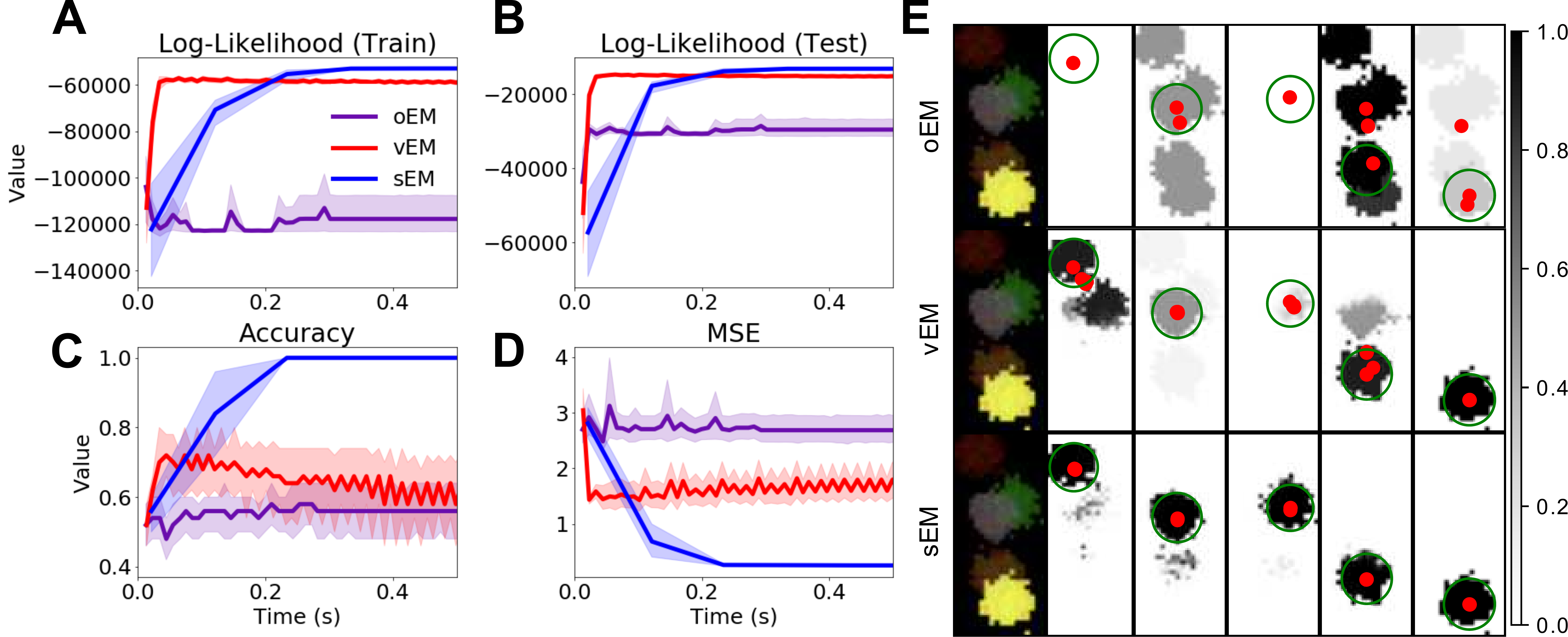}
    \caption{Performance evaluation of sEM, vEM and oEM on \textit{C. elegans} segmentation. Optimization configuration: \textbf{Random center initialization, covariance update enabled, random covariance initialization}. \textbf{A-D}. Training (\textbf{A}) and test (\textbf{B}) log-likelihoods, segmentation accuracy (\textbf{C}) and mean squared error (\textbf{D}) for the three methods. \textbf{E}. The visual segmentation quality. Each row denotes a different method; the first column shows the observed neuronal pixel values (identical for all three methods), and the remaining columns indicate the mean identified segmentation of each neuron (in grayscale heatmaps) and the inferred cell center in red dots, over multiple randomized runs. The ground truth neuron shape is overlaid in green.}
    \label{fig:worm_TTF_config}
\end{figure}
\begin{figure}[H]
    \centering
    \includegraphics[width=1\linewidth]{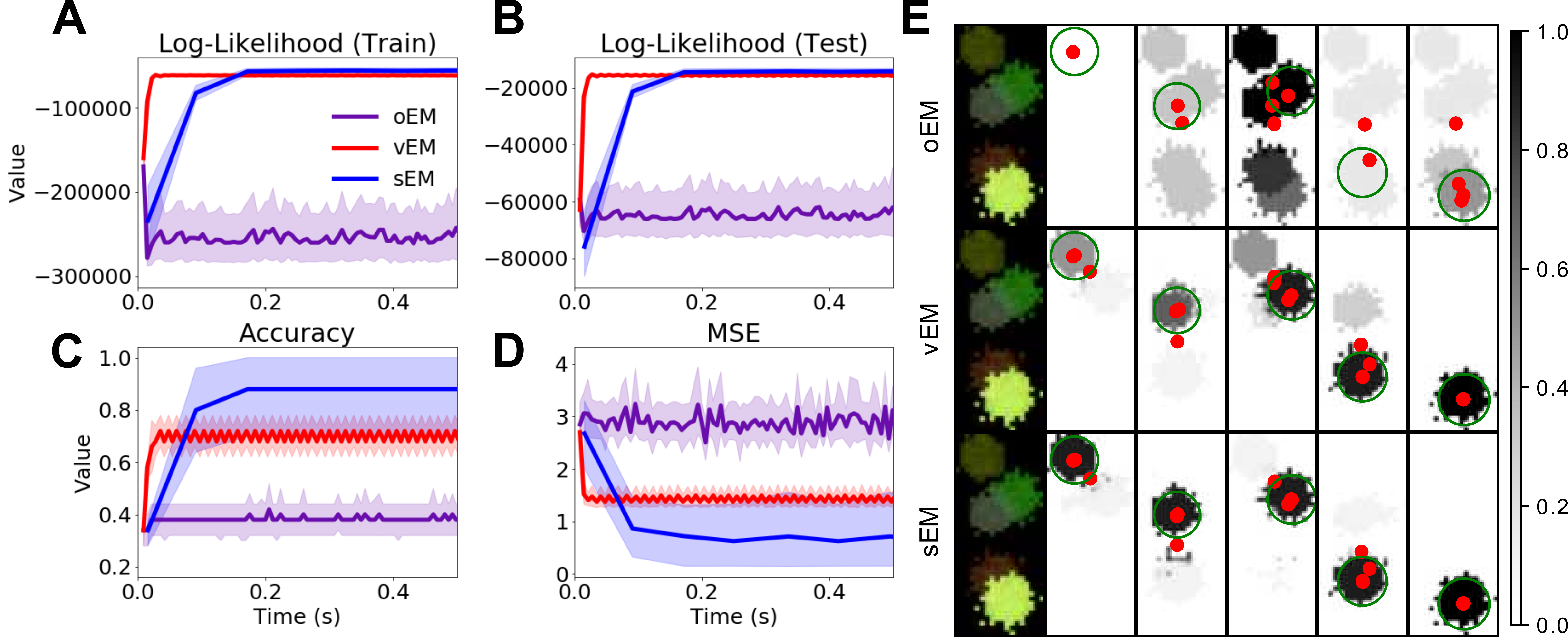}
        \caption{Performance evaluation of sEM, vEM and oEM on \textit{C. elegans} segmentation. Optimization configuration: \textbf{Random center initialization, covariance update disabled, ground truth covariance initalization}. \textbf{A-D}. Training (\textbf{A}) and test (\textbf{B}) log-likelihoods, segmentation accuracy (\textbf{C}) and mean squared error (\textbf{D}) for the three methods. \textbf{E}. The visual segmentation quality. Each row denotes a different method; the first column shows the observed neuronal pixel values (identical for all three methods), and the remaining columns indicate the mean identified segmentation of each neuron (in grayscale heatmaps) and the inferred cell center in red dots, over multiple randomized runs. The ground truth neuron shape is overlaid in green.}
    \label{fig:worm_TFT_config}
\end{figure}
\begin{figure}[H]
    \centering
    \includegraphics[width=1\linewidth]{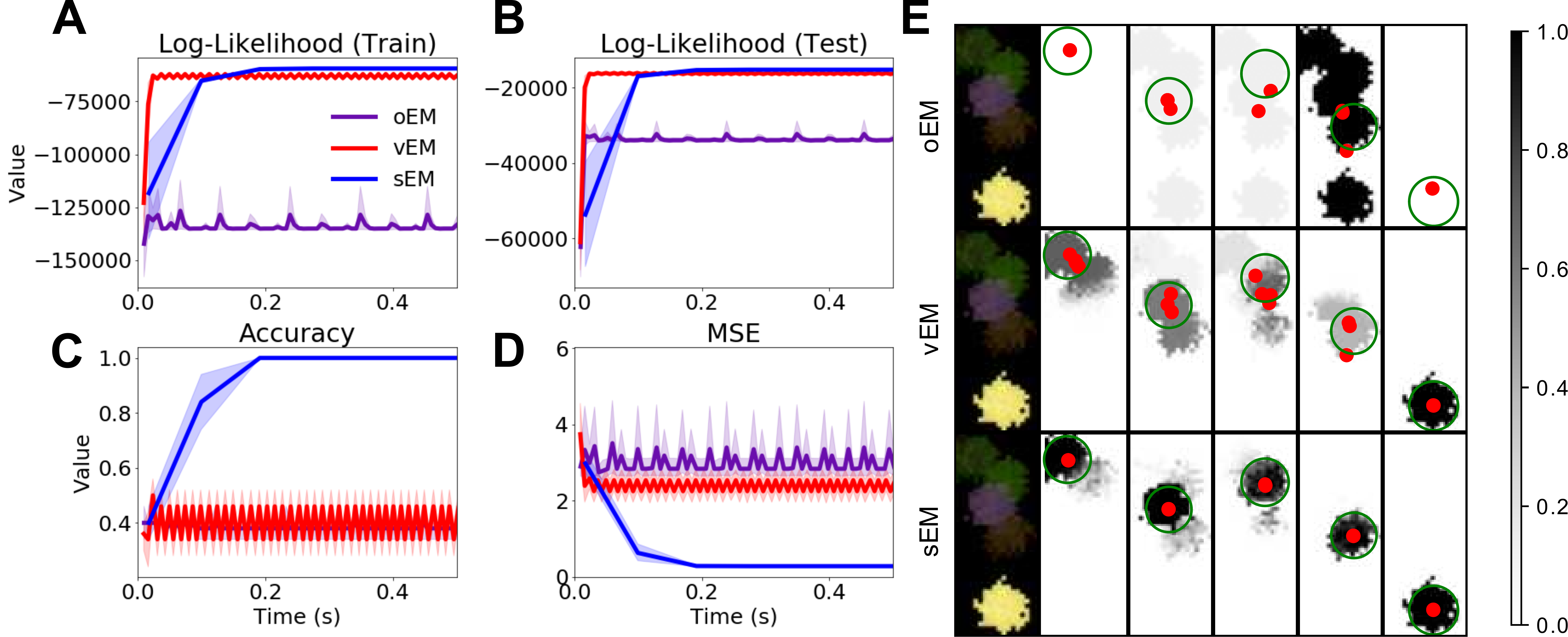}
    \caption{Performance evaluation of sEM, vEM and oEM on \textit{C. elegans} segmentation. Optimization configuration: \textbf{Random center initialization, covariance update disabled, random covariance initalization}. \textbf{A-D}. Training (\textbf{A}) and test (\textbf{B}) log-likelihoods, segmentation accuracy (\textbf{C}) and mean squared error (\textbf{D}) for the three methods. \textbf{E}. The visual segmentation quality. Each row denotes a different method; the first column shows the observed neuronal pixel values (identical for all three methods), and the remaining columns indicate the mean identified segmentation of each neuron (in grayscale heatmaps) and the inferred cell center in red dots, over multiple randomized runs. The ground truth neuron shape is overlaid in green.}
    \label{fig:worm_TFF_config}
\end{figure}

\begin{figure}[H]
    \centering
    \includegraphics[width=1\linewidth]{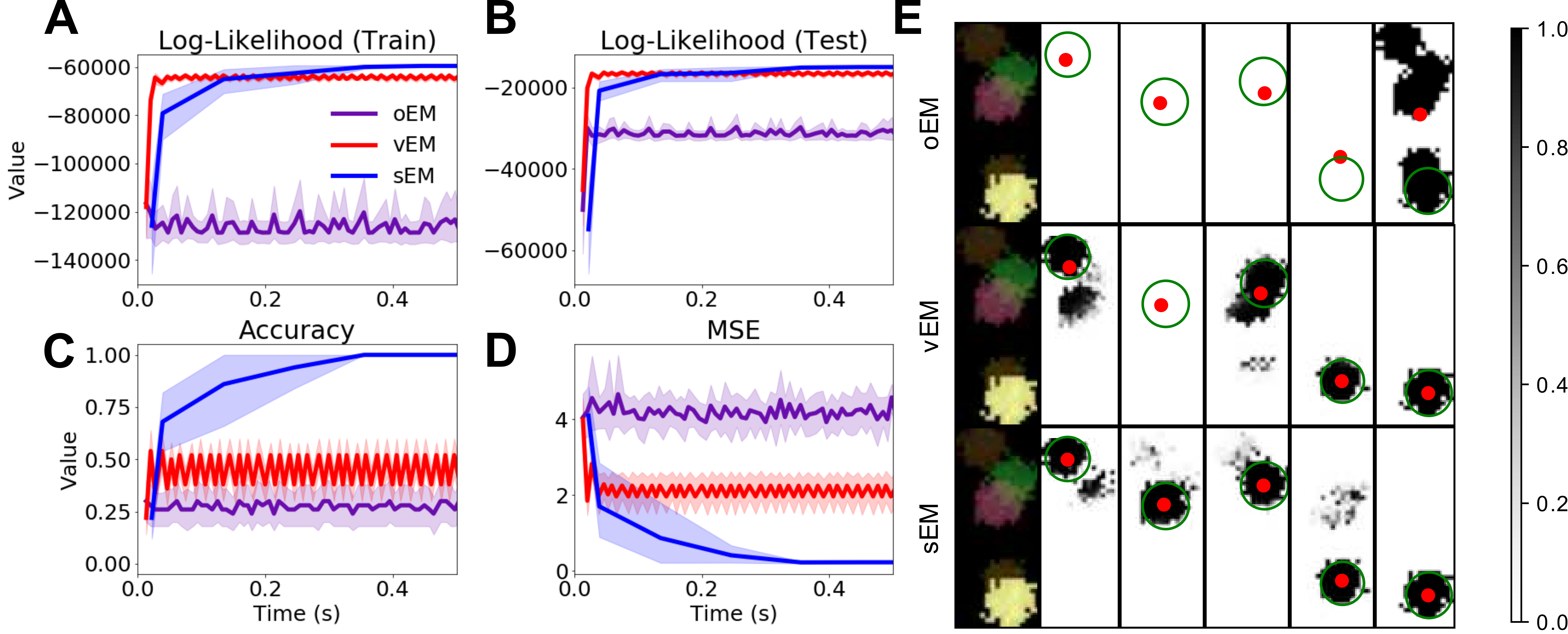}
    \caption{Performance evaluation of sEM, vEM and oEM on \textit{C. elegans} segmentation. Optimization configuration: \textbf{Atlas prior center initialization, covariance update disabled, random covariance initalization}. \textbf{A-D}. Training (\textbf{A}) and test (\textbf{B}) log-likelihoods, segmentation accuracy (\textbf{C}) and mean squared error (\textbf{D}) for the three methods. \textbf{E}. The visual segmentation quality. Each row denotes a different method; the first column shows the observed neuronal pixel values (identical for all three methods), and the remaining columns indicate the mean identified segmentation of each neuron (in grayscale heatmaps) and the inferred cell center in red dots, over multiple randomized runs. The ground truth neuron shape is overlaid in green.}
    \label{fig:worm_FFF_config}
\end{figure}

\begin{figure}[H]
    \centering
    \includegraphics[width=1\linewidth]{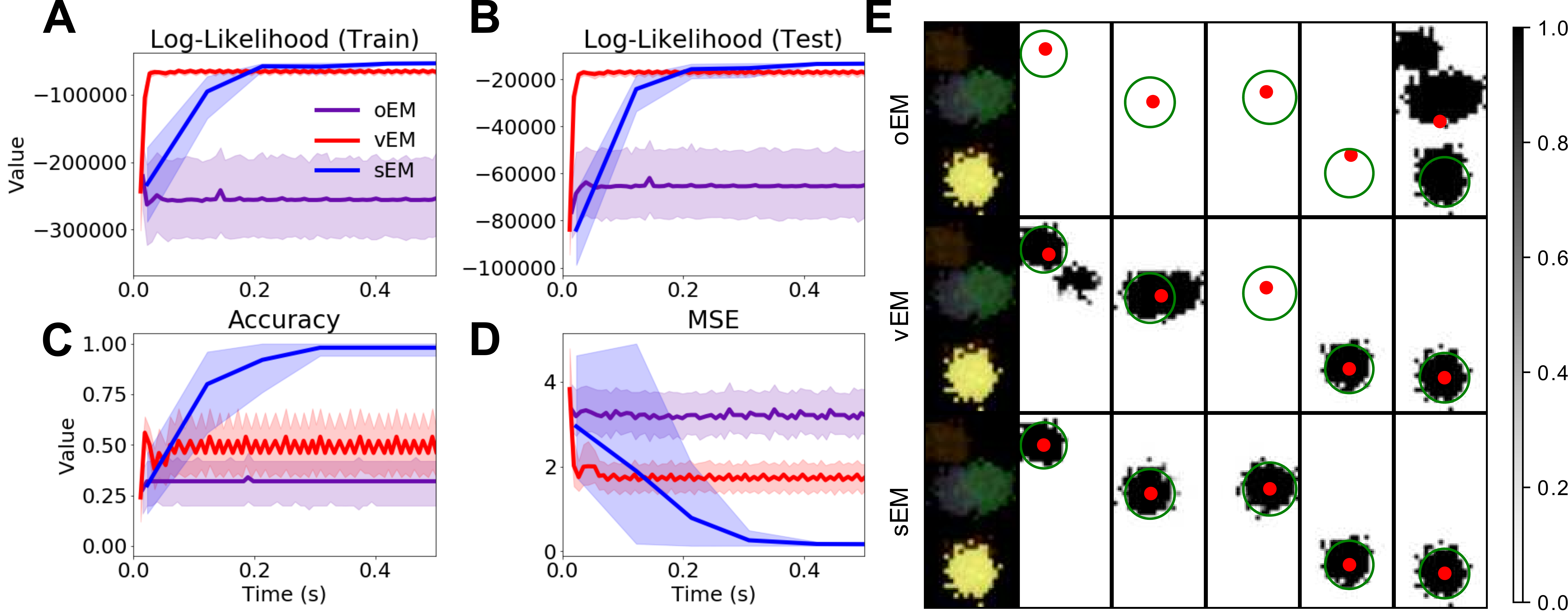}
    \caption{Performance evaluation of sEM, vEM and oEM on \textit{C. elegans} segmentation. Optimization configuration: \textbf{Atlas prior center initialization, covariance update disabled, ground truth covariance initalization}. \textbf{A-D}. Training (\textbf{A}) and test (\textbf{B}) log-likelihoods, segmentation accuracy (\textbf{C}) and mean squared error (\textbf{D}) for the three methods. \textbf{E}. The visual segmentation quality. Each row denotes a different method; the first column shows the observed neuronal pixel values (identical for all three methods), and the remaining columns indicate the mean identified segmentation of each neuron (in grayscale heatmaps) and the inferred cell center in red dots, over multiple randomized runs. The ground truth neuron shape is overlaid in green.}
    \label{fig:worm_FFT_config}
\end{figure}

\begin{figure}[H]
    \centering
    \includegraphics[width=1\linewidth]{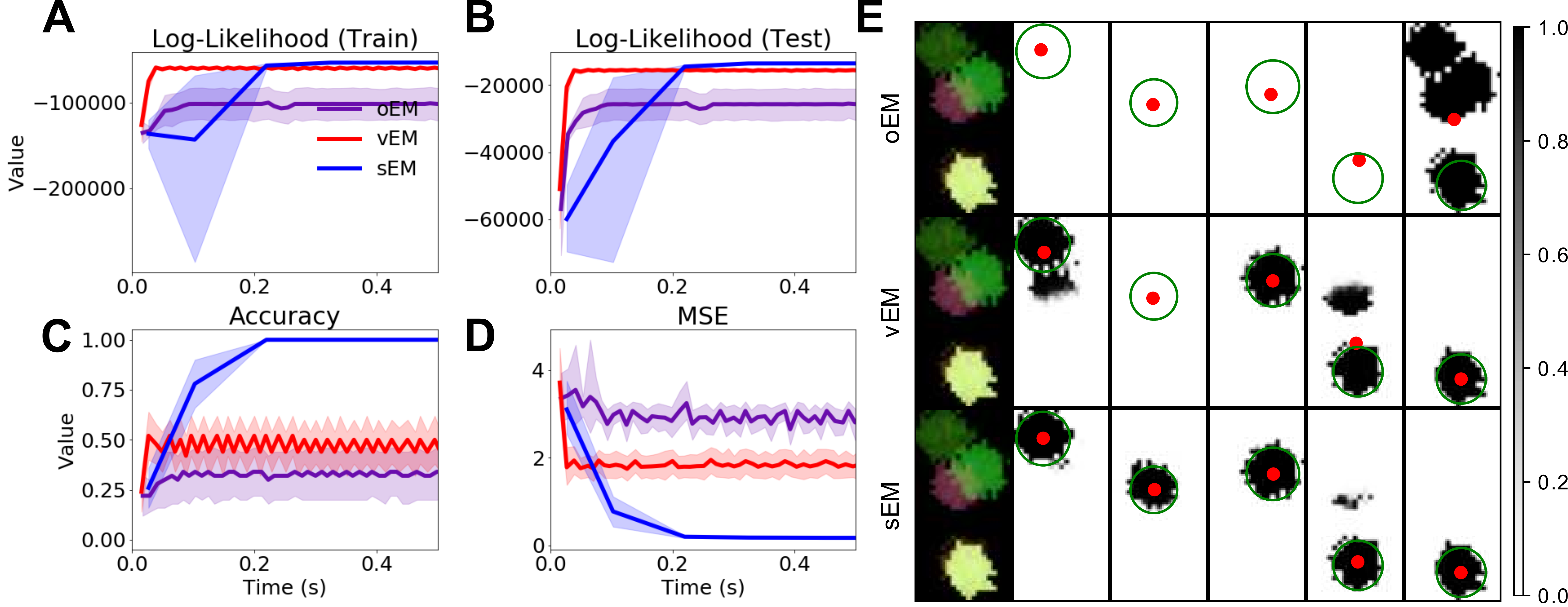}
    \caption{Performance evaluation of sEM, vEM and oEM on \textit{C. elegans} segmentation. Optimization configuration: \textbf{Atlas prior center initialization, covariance update enabled, random covariance initalization}. \textbf{A-D}. Training (\textbf{A}) and test (\textbf{B}) log-likelihoods, segmentation accuracy (\textbf{C}) and mean squared error (\textbf{D}) for the three methods. \textbf{E}. The visual segmentation quality. Each row denotes a different method; the first column shows the observed neuronal pixel values (identical for all three methods), and the remaining columns indicate the mean identified segmentation of each neuron (in grayscale heatmaps) and the inferred cell center in red dots, over multiple randomized runs. The ground truth neuron shape is overlaid in green.}
    \label{fig:worm_FTF_config}
\end{figure}

\bibliographystyle{apalike}
\bibliography{bibliography}

\end{document}